\pdfoutput=1 % MUST stay within the first 5 lines: tells arXiv to use pdfLaTeX.
\documentclass[11pt]{article}

\usepackage[font=libertinus, citestyle=numeric]{kurbanlab}
\DeclareAffiliation{hbku}{%
  College of Science and Engineering, Hamad Bin Khalifa University, Doha, Qatar}

\DeclareAffiliation{tamu}{%
  Department of Computer and Electrical Engineering,
  Texas A\&M University, College Station, TX, USA}

\DeclareAffiliation{iub}{%
  Luddy School of Informatics, Computing, and Engineering,
  Indiana University Bloomington, Bloomington, IN, USA}

\DeclareAffiliation{tamuq}{Department of Electrical and Computer Engineering, Texas A\&M University at Qatar, Doha, Qatar}
\title{Conformal Coverage Guarantees\\
for Any Video Temporal Grounder}
\Subtitle{\cover{}: a post-hoc wrapper that turns any grounder into one that emits a temporal
region containing the true moment with probability at least $1-\alpha$.}
\RunningTitle{Conformal Coverage Guarantees
for Any Video Temporal Grounder}

\Author{Aseel Mohamed}{tamuq}
\Author{Rasul Khanbayov}{hbku}
\Author{Erchin Serpedin}{tamu}
\Author[corresponding=hkurban@hbku.edu.qa, orcid=0000-0003-3142-2866]{Hasan Kurban}{hbku}
\Keywords{video temporal grounding; black-box model; localizer; wrapper }
\CodeURL{https://github.com/KurbanIntelligenceLab/cover}
\Venue{Submitted to AAAI'27}          % or: \Venue{To appear at NeurIPS 2026}
\begin{document}
\maketitle

\begin{abstract}
Event boundaries in continuous video are ambiguous: re-annotate the same
query--video pair and independent annotators mark moments that overlap by less
than half on a large fraction of samples. The ground truth for video temporal
grounding is therefore a distribution over intervals, yet every grounder returns
a single interval with no statement of reliability, so at deployment a wrong
interval is indistinguishable from a right one. \cover{} changes the output
object: a post-hoc, model-agnostic wrapper that turns any grounder, a trained
localizer or a black-box video--language model, into one that emits a temporal
region containing the true moment with probability at least $1-\alpha$, by
calibrating the quantile of a temporal nonconformity score on held-out labels and
widening the base prediction by that amount. The guarantee is finite-sample and
distribution-free under exchangeability, and requires neither retraining nor
white-box access. We give two score families, a two-sided boundary-widening score
for grounders that emit an interval and a super-level-set score for grounders
that emit a relevance signal, and develop theory specific to grounding that
bounds how large the certified region becomes, when coverage survives
conditioning on event length, and how it degrades when moments from one video
break exchangeability. Across three benchmarks and five grounders, realized
coverage tracks the target,
and calibration exposes what point metrics hide. An evidential grounder
purpose-built for uncertainty quantification realizes only $0.660$ against its
own nominal $0.80$ target, while \cover{} restores a valid guarantee on the
identical predictions. A video--language model's event onsets prove essentially
free of error while all of its boundary error concentrates on offsets. And
hand-picked fixed margins swing realized coverage by up to $0.904$, so the
calibration step is what makes a region mean anything.
\end{abstract}

\printkeywords

%%%%%%%%%%%%%%%%%%%%%%%%%%%%%%%%%%%%%%%%%%%%%%%%%%%%%%%%%%%%%%%%%%%%%%%%%%%%%%%%

\begin{figure}[t]
  \centering
  \kilgraphics{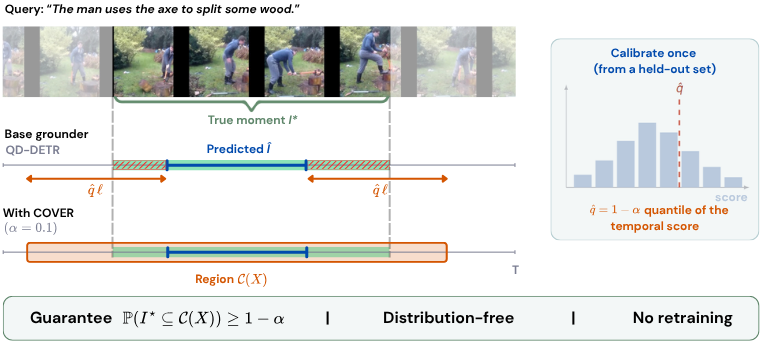}
  \caption{\textbf{\cover{} turns any grounder into one with a coverage guarantee.}
  A base grounder returns a single interval $\Ihat$ (blue) that can miss a part of the true moment $\Igt$ (green). The missed tail (red) comes with no warning. \cover{} reads off, once on a held-out calibration set, the $(1-\alpha)$
quantile $\qhat$ of a temporal nonconformity score (right), then widens every
prediction by $\qhat\,\ell$ per side into a region $\Cset(X)$ (orange) that
contains $\Igt$ with probability at least $1-\alpha$. The guarantee is
distribution-free, needs no retraining, and holds for trained localizers and
black-box VLMs alike.}
  \label{fig:teaser}
\end{figure}

\section{Introduction}
\label{sec:intro}

Ask several people to mark when an event begins and ends in a continuous video
and they will not agree. Re-annotate the same query--video pair and the marked
moments overlap by less than half on a large fraction of samples
\citep{otani2020hidden,dtgspl2023}, an ambiguity surveys name a core open problem
of the task \citep{zhang2023tsgv}. It is a property of continuous video, not an
artifact of any model, so the ground truth for temporal grounding is a
distribution over intervals rather than a point.

Grounders report a point anyway. A temporal grounder takes a video and a
language query and returns the interval where the queried event occurs
\citep{gao2017tall,krishna2017anet,zhang2023tsgv}, and modern systems, whether
trained localizers \citep{lei2021qvhighlights,moon2023qddetr,lee2024bamdetr,
lin2023univtg} or video--language models prompted for timestamps
\citep{huang2023vtimellm,zheng2024tfvtg,f2g2026}, return one interval and stop.
The mismatch is not cosmetic: in deployment a confident interval that is wrong
looks exactly like a confident interval that is right, and any downstream
consumer, a retrieval index, an editing tool, or an evidence-seeking agent, has
no principled way to decide when to widen its search or defer.

The right output object is a region carrying a statement of the form ``the true
moment lies here with probability at least $1-\alpha$.'' Two routes to one
suggest themselves, and neither delivers. The first is to pad the predicted
interval by a margin, but no margin is knowable in advance: across six
hand-picked constants we measure realized coverage anywhere from $0.096$ to
$1.000$, and a padding adequate on one dataset misses its target by more than
$0.3$ on another (Table~\ref{tab:e1-uncalibrated}). The second is to train the
uncertainty in, as methods that learn a boundary variance or an evidential
distribution do \citep{ma2024beyond,evidential_mr2025}. These are tied to
one architecture, require retraining, and cannot wrap a grounder served behind a
text-only API, and they do not always deliver the reliability they report, as an
evidential grounder's own predicted boundary uncertainty, converted to a region
at its nominal $0.80$ level, realizes only $0.660$ (Appendix Table \ref{tab:e1-SRAM}).

Conformal prediction supplies precisely the missing guarantee, distribution-free
finite-sample coverage from a held-out calibration set
\citep{vovk2005algorithmic,angelopoulos2023gentle,lei2018distribution,
angelopoulos2024foundations}. It has been instantiated for classification and
regression \citep{angelopoulos2021uncertainty,romano2019cqr}, for spatial object
detection and segmentation \citep{mukama2024copula,bates2021rcps,
andeol2025seqcrc,timans2024adaptivebox,instanceseg2026cp}, and recently for video
question answering \citep{videoqa_abstention2026}, but not for the object
grounding produces: the start and end of a moment.

\cover{} supplies it with a post-hoc wrapper (Figure~\ref{fig:teaser}). We define
a temporal nonconformity score equal to the smallest widening of the base
interval that makes it contain the true moment, take the $(1-\alpha)$ empirical
quantile of that score over a labeled calibration set, and output every test
prediction widened by that quantile. Exchangeability of calibration and test sets
then gives coverage at least $1-\alpha$. Because the
wrapper touches the grounder only through its outputs and a held-out label set,
it applies unchanged to a trained DETR-style localizer and to a black-box
video--language model. Two score families cover the two grounder types, and for
both we prove marginal coverage with a matching upper bound against gross
over-coverage, and give a risk-control variant that bounds the expected miss-rate
at a user level with high probability over the calibration draw
\citep{bates2021rcps,angelopoulos2024crc}.

\paragraph{Contributions} We make the following contributions:
\begin{enumerate}
  \item \textbf{A region, not a point, for a task whose ground truth is a
distribution.} The first distribution-free uncertainty quantification for video
temporal grounding: a post-hoc wrapper that attaches a finite-sample coverage
guarantee to \emph{any} grounder, trained or black-box, by calibrating a single
scalar on held-out labels, with no retraining and no white-box access.
  \item \textbf{Temporal nonconformity scores for both kinds of grounder.} A
two-sided boundary-widening score for grounders that emit an interval, and a
super-level-set score for grounders that expose a relevance signal, the latter
returning regions that spend temporal budget only where the grounder
places relevance. Where boundary error is asymmetric, a two-parameter
per-boundary variant calibrates each side separately.
\item \textbf{Theory that says what the region costs, not only that it covers}: an efficiency identity with asymptotic length-optimality of
the calibrated widening, length-conditional coverage of the normalized score
under a scale model, and an explicit coverage bound under the within-video dependence
that breaks exchangeability. Standard coverage, tightness, and risk-control
results are in the Appendix.
\item \textbf{Certification as measurement.} Across three benchmarks and five
grounders, realized coverage tracks the target, and calibration exposes what
point metrics hide: an evidential grounder built for uncertainty quantification
misses its nominal coverage by $14$ points while \cover{} restores a valid
guarantee on the identical predictions, a VLM's event onsets are essentially free
while all of its boundary error sits on offsets, and hand-picked margins swing
realized coverage range by up to $0.904$.
\end{enumerate}

\begin{kilkey}
\cover{} is a post-hoc, model-agnostic wrapper that turns any grounder, a trained
localizer or a black-box video--language model, into one that emits a temporal
region containing the true moment with probability at least $1-\alpha$.
\end{kilkey}

%%%%%%%%%%%%%%%%%%%%%%%%%%%%%%%%%%%%%%%%%%%%%%%%%%%%%%%%%%%%%%%%%%%%%%%%%%%%%%%%
\section{Related work}
\label{sec:related}

\paragraph{Uncertainty in temporal grounding}
Prior uncertainty-aware grounders learn a distribution over boundaries, for
example Gaussian boundary variances or evidential heads trained jointly with the
localizer \citep{ma2024beyond,huang2022uncertainty}, an approach still being
actively developed for moment retrieval \citep{evidential_mr2025}. They are
motivated by the same observation we are, that event boundaries in continuous
video are inherently ambiguous, an effect measured directly when the same
query--video pair is re-annotated by multiple people and their moments overlap by
less than half on many samples \citep{otani2020hidden,dtgspl2023}. But they
quantify uncertainty in an architecture-specific way, require retraining, and
offer no finite-sample coverage statement, so a user cannot set a target error
rate and have it met. \cover{} instead is post-hoc and distribution-free: it
wraps a fixed grounder and delivers a coverage level the user chooses, which is a
different guarantee and a different deployment model.

\paragraph{Conformal prediction and risk control}
Conformal prediction turns any point predictor into a set predictor with
distribution-free, finite-sample marginal coverage under exchangeability
\citep{vovk2005algorithmic,lei2018distribution,angelopoulos2023gentle,
angelopoulos2024foundations}. Conformalized quantile regression calibrates two-sided regression intervals \citep{romano2019cqr}, and risk-controlling prediction sets, conformal risk control, and Learn-then-Test extend the guarantee from coverage to losses monotone in the set-size parameter, such as false-negative rate and intersection-over-union
\citep{bates2021rcps,angelopoulos2024crc,angelopoulos2021ltt}. The same tools now underpin a fast-growing line on reliability for language and
vision--language models, where conformal risk control and selective conformal
prediction bound answer-set error or drive abstention
\citep{videoqa_abstention2026}. That line certifies which answer to return or
whether to answer at all, a different output object than the temporal region we
certify. Our contribution is not a new calibration machinery but the object it is
pointed at: the temporal region, the scores that make intervals and relevance
signals conformable, the theory governing how large the region gets, and what
certifying real grounders reveals about them. Exact conditional coverage is
impossible distribution-free \citep{barber2021limits}. Mondrian/group-conditional
calibration \citep{vovk2012conditional} and recent conditional-guarantee methods
\citep{gibbs2024conditional} are the practical response we adopt for length and
class strata. When exchangeability fails, across clips of one video or across
datasets, weighting and adaptive variants restore validity
\citep{tibshirani2019covariate,barber2023beyond,gibbs2021aci}.

\paragraph{Conformal methods in vision and video}
The spatial analogue is established: conformal prediction calibrates a margin on
predicted boxes or masks so the object is contained at a guaranteed rate
\citep{mukama2024copula,andeol2023railway,timans2024adaptivebox,instanceseg2026cp},
and recent work certifies the full detection pipeline through \emph{sequential,
two-parameter} risk control for any backbone including DETR \citep{andeol2025seqcrc}.
That two-parameter structure is forced by the detector: separate thresholds for
confidence filtering and for localization, mediated by non-maximum suppression.
Grounding poses a different problem. There is no object set
and no suppression step, but the target is a continuous extent whose endpoints
are ambiguous in the data itself, so the difficulty moves from deciding
\emph{which} predictions to certify to deciding \emph{how far} a single one must
reach, and the interesting structure is temporal rather than combinatorial: how
region length scales with event length, whether coverage survives conditioning on
duration, and what happens when moments drawn from one video cease to be
exchangeable. The closest neighbors to position against are therefore the
conformal object-detection line \citep{andeol2025seqcrc} and the video-QA
reliability line \citep{videoqa_abstention2026}. Because moments come from continuous video, the exchangeability
question also links \cover{} to conformal time-series forecasting
\citep{stankeviciute2021cts,barber2023beyond}.

%%%%%%%%%%%%%%%%%%%%%%%%%%%%%%%%%%%%%%%%%%%%%%%%%%%%%%%%%%%%%%%%%%%%%%%%%%%%%%%%
\section{Problem Setup}
\label{sec:setup}

\paragraph{Grounder as a black box}
A grounder $g$ maps an input $X=(\text{video},\text{query})$ to a predicted
interval $\Ihat(X)=[\shat,\ehat]\subseteq[0,T]$, and optionally to a temporal
relevance signal $f_X:[0,T]\to\R$ (a saliency or per-clip score). The true moment
is $\Igt=[\sstar,\estar]$. We treat $g$ as fixed and access it only through
$\Ihat$ (and $f_X$ when available). Nothing about $g$ is assumed or retrained.

\paragraph{Goal.}
For a user-chosen $\alpha\in(0,1)$, output a temporal region $\Cset(X)\subseteq
[0,T]$ such that
\[
\Prob\big(\Igt \subseteq \Cset(X)\big)\;\ge\;1-\alpha,
\]
where the probability is over a fresh draw of $(X,\Igt)$ and the calibration set.
Containment of the whole true interval is the natural grounding analogue of
coverage: the region is trusted to hold the event, so missing either boundary is
a failure. We also consider the weaker miss-rate objective in
\Cref{sec:theory}.

\paragraph{Calibration data}
We assume a labeled calibration set $\{(X_i,\Igt_i)\}_{i=1}^{n}$ drawn
exchangeably with the test point, disjoint from anything used to build $g$. This
reuses existing ground-truth intervals, it adds no new annotation. The wrapper's
only fitted quantity is a scalar threshold.

%%%%%%%%%%%%%%%%%%%%%%%%%%%%%%%%%%%%%%%%%%%%%%%%%%%%%%%%%%%%%%%%%%%%%%%%%%%%%%%%
\section{Method: \cover{}}
\label{sec:method}

\cover{} chooses a nested family of candidate regions $\Cset_\lambda(X)$ indexed
by a scalar $\lambda$ (larger $\lambda$ gives a larger region and a negative
calibrated $\lambda$ tightens the base interval when it already over-covers), defines a
nonconformity score equal to the smallest $\lambda$ that covers the truth,
calibrates the score's quantile, and returns the region at that quantile
(Figure \ref{fig:method} in the Appendix). Two
instantiations follow, and both reduce to a one-dimensional split-conformal
calibration.

\paragraph{Interval-widening score}
For grounders that emit an interval, widen per side by a fixed scale (we use the
predicted length $\ell(X)=\ehat-\shat$ for adaptivity, or $1$ for plain seconds):
\[
\Cset_\lambda(X)=\big[\shat-\lambda\,\ell(X),\;\ehat+\lambda\,\ell(X)\big]\cap[0,T].
\]
The smallest $\lambda$ that makes $\Cset_\lambda(X)\supseteq\Igt$ is
\begin{equation}
\score(X,\Igt)=\max\!\left(\frac{\shat-\sstar}{\ell(X)},\;
\frac{\estar-\ehat}{\ell(X)}\right),
\label{eq:score-interval}
\end{equation}
the larger of the two per-side overshoots (negative when $\Ihat$ already strictly
contains $\Igt$). This is the two-sided analogue of conformalized regression
\citep{romano2019cqr} applied to the start and end boundaries jointly. Taking the
maximum over both sides keeps the score a single scalar, so one quantile delivers
joint two-sided coverage with no per-boundary correction and no second split of
the calibration data.

\paragraph{Per-boundary variant}
The maximum in \eqref{eq:score-interval} spends the same widening on both sides,
which is efficient when the two boundaries are comparably hard and wasteful when
they are not. A grounder that already locates event onsets well pays for that
side anyway. The two-parameter variant instead scores each boundary separately,
$\score^{\mathrm{s}}=(\shat-\sstar)/\ell(X)$ and
$\score^{\mathrm{e}}=(\estar-\ehat)/\ell(X)$, calibrates
$\qhat_{\mathrm{s}},\qhat_{\mathrm{e}}$ on their own quantiles as in
conformalized quantile regression \citep{romano2019cqr}, and returns
$[\shat-\qhat_{\mathrm{s}}\ell,\;\ehat+\qhat_{\mathrm{e}}\ell]\cap[0,T]$. It buys
per-side adaptivity at the cost of a second calibrated parameter, and which
member wins is an empirical property of the grounder, diagnosed by the calibrated
$\qhat_{\mathrm{s}}/\qhat_{\mathrm{e}}$ ratio.

\paragraph{Super-level-set score}
When the grounder exposes a temporal relevance signal $f_X$, take
$\Cset_\lambda(X)=\{t:f_X(t)\ge -\lambda\}$, a union of intervals. Covering all of
$\Igt$ requires $f_X(t)\ge-\lambda$ for every $t\in\Igt$, so the binding score is
\begin{equation}
\score(X,\Igt)=-\min_{t\in\Igt} f_X(t).
\label{eq:score-sls}
\end{equation}
This yields adaptive, possibly disconnected regions: when $f_X$ is multimodal the
region splits, spending temporal budget only where the grounder places relevance.

\paragraph{Split-conformal calibration}
Given either score, compute $\score_i=\score(X_i,\Igt_i)$ on the calibration set
and set the threshold to the finite-sample-corrected empirical quantile
\begin{equation}
\qhat=\Big\lceil (n+1)(1-\alpha)\Big\rceil\text{-th smallest of }
\{\score_1,\dots,\score_n\}.
\label{eq:qhat}
\end{equation}
At test, output $\Cset_{\qhat}(X)$ (Algorithm \ref{alg:cover} in the Appendix). The only
computation beyond running $g$ is sorting $n$ scalars, so the wrapper is
negligible in cost and adds no parameters to $g$.

\paragraph{Risk control and stratified calibration}
Two variants extend the wrapper. When the user wants the expected miss-rate held at
$\alpha$ with confidence $1-\delta$ over the calibration draw rather than full
coverage, we select $\lambda$ by a risk-controlling procedure \citep{bates2021rcps}:
the loss $\mathcal{L}_\lambda(X,\Igt)=\mathbf{1}[\score(X,\Igt)>\lambda]$ is monotone
in $\lambda$, so an upper-confidence bound on the empirical risk yields a
$\hat\lambda$ with $\Prob(\E[\mathcal{L}_{\hat\lambda}]\le\alpha)\ge 1-\delta$
(Proposition \ref{prop:rcps} in the Appendix). When marginal coverage hides undercoverage of a
subpopulation (short events, a rare class), calibrating a separate $\qhat$ per group
\citep{vovk2012conditional} restores coverage within each group at the cost of
splitting the calibration data. Exact attribute-free conditional coverage is not
attainable distribution-free \citep{barber2021limits}.

%%%%%%%%%%%%%%%%%%%%%%%%%%%%%%%%%%%%%%%%%%%%%%%%%%%%%%%%%%%%%%%%%%%%%%%%%%%%%%%%

\section{Analysis}
\label{sec:theory}

The marginal coverage guarantee is
the standard split-conformal result (Proposition \ref{prop:cov}). The rest of this
section is theory specific to temporal grounding, how large the regions are, when
coverage holds within event-length strata, and what happens when moments from one
video break exchangeability. Each result is stated in full here; the complete
proofs, along with standard coverage, tightness, and risk-control facts, are in
the Appendix.

\begin{assumption}[Exchangeability]\label{a:exch}
The calibration points $(X_i,\Igt_i)_{i=1}^n$ and the test point
$(X_{n+1},\Igt_{n+1})$ are exchangeable.
\end{assumption}
\begin{assumption}[Nested family]\label{a:nested}
$\Cset_\lambda$ is nested: $\lambda\le\lambda'\Rightarrow
\Cset_\lambda(X)\subseteq\Cset_{\lambda'}(X)$, and
$\Igt\subseteq\Cset_\lambda(X)\Leftrightarrow \score(X,\Igt)\le\lambda$.
\end{assumption}
\begin{assumption}[Frozen pipeline]\label{a:frozen}
The grounder $g$ and the score $\score$ are fixed before the calibration labels
are seen.
\end{assumption}

\begin{proposition}[Marginal coverage: foundation]\label{prop:cov}
Under Assumptions \ref{a:exch}--\ref{a:frozen}, Algorithm \ref{alg:cover} of the Appendix with
$\qhat$ from \eqref{eq:qhat} satisfies
$\Prob(\Igt_{n+1}\subseteq \Cset_{\qhat}(X_{n+1}))\ge 1-\alpha$.
\end{proposition}

\noindent This is split conformal \citep{lei2018distribution} with the temporal
score. The same data bound coverage above by $1-\alpha+\tfrac1{n+1}$
(Appendix Proposition \ref{prop:tight}) and control the expected miss-rate at level $\alpha$
with probability $1-\delta$ (Appendix Proposition \ref{prop:rcps}).

\paragraph{How large is the region?}
Let $S=\score(X,\Igt)$ denote the score as a random variable, $F_S$ its
distribution function, and $Q_{1-\alpha}(F_S)=\inf\{\lambda:F_S(\lambda)\ge
1-\alpha\}$ its $(1-\alpha)$ quantile. Coverage is worthless without a handle on
region size. The next result gives one and shows the calibrated widening is as
tight as the family allows, for both score families at once.

\begin{theorem}[Efficiency and length-optimality]\label{thm:eff}
Fix the interval score \eqref{eq:score-interval} with predicted length
$\ell(X)=\ehat-\shat>0$, under Assumptions \ref{a:exch}--\ref{a:frozen}.
\emph{(i) Length identity.} The region length obeys
$|\Cset_\lambda(X)|=\min\{\ell(X)(1+2\lambda),\,T\}\le\ell(X)(1+2\lambda)$, so
$\E|\Cset_{\qhat}(X)|\le\E[\ell(X)(1+2\qhat)]$, with equality to
$\E[\ell(X)](1+2\,\E[\qhat])$ when the test point is independent of the
calibration data.
\emph{(ii) Consistency (i.i.d.\ calibration).} If the calibration scores are
i.i.d.\ (a strengthening of Assumption \ref{a:exch}), then $\qhat\to
Q_{1-\alpha}(F_S)$ almost surely as $n\to\infty$; if in addition $F_S$ has a
positive density at $Q_{1-\alpha}(F_S)$, then
$\sqrt{n}\,(\qhat-Q_{1-\alpha}(F_S))=O_{\mathrm p}(1)$.
\emph{(iii) Family length-optimality.} For continuous $F_S$, among the
one-parameter family $\{\Cset_\lambda\}_{\lambda\ge0}$ the region
$\Cset_{Q_{1-\alpha}(F_S)}$ is the shortest with marginal coverage at least
$1-\alpha$: every $\lambda<Q_{1-\alpha}(F_S)$ has coverage below $1-\alpha$ and
every $\lambda>Q_{1-\alpha}(F_S)$ yields a strictly longer region.
\emph{(iv) The same holds for the relevance signal.} For the super-level-set
score \eqref{eq:score-sls} and the threshold family
$\{t:f_X(t)\ge\tau\}_{\tau\in\R}$, family coverage equals $F_S(-\tau)$ and
\cover{} uses $\tau=-\qhat$; since the temporal measure
$\E|\{t:f_X(t)\ge\tau\}|$ is also monotone in $\tau$, the threshold
$\tau=-Q_{1-\alpha}(F_S)$ gives the smallest expected-measure region among all
fixed-threshold rules on the signal at the target coverage.
\end{theorem}
% \begin{proof}[Proof sketch]
% Clipping to $[0,T]$ only shortens the bracket of unclipped length $\ell(1+2\lambda)$,
% giving (i), with the equality from independence. For (ii)--(iv) the mechanism is
% one and the same: family coverage is a monotone function of the parameter
% ($F_S(\lambda)$, nondecreasing; $F_S(-\tau)$, nonincreasing) while region size
% moves strictly the other way, so the coverage constraint binds exactly at the
% quantile and the extremal valid parameter minimizes size. $\qhat$ converges to
% that quantile by Glivenko--Cantelli, with the $\sqrt{n}$ rate from the Bahadur
% representation under the density condition. Full proofs in the Appendix.
% \end{proof}
\begin{remark}[Scope of the optimality claims]\label{rem:scope}
Theorem \ref{thm:eff}(iii)--(iv) asserts optimality \emph{within} a
one-parameter family, not over all coverage regions: an input-varying width can be
shorter at the same coverage, and the per-boundary variant of \cref{sec:method}
is one such member. The finite-sample cost of estimating the quantile is
characterized exactly by Proposition \ref{prop:beta} in the Appendix, whose $O(n^{-1/2})$ spread
matches the resampled coverage variability in \cref{sec:Results}.
\end{remark}

\paragraph{When is coverage length-conditional?}
Marginal coverage is an average, and an average can hide a badly undercovered
subpopulation behind an overcovered one. Event length is the stratification that
matters most in grounding, and the choice of scale $\ell(X)$ in
\eqref{eq:score-interval} is exactly what decides whether coverage survives it.
The next result says the normalized score is length-conditional under a scale
model and the seconds-valued score is not. The prediction is directional and
therefore testable (Appendix Table \ref{tab:e3-sec}).

\begin{proposition}[Length-conditional coverage under a scale model]\label{prop:lcc}
Assume i.i.d.\ sampling and the \emph{scale model}: the normalized score
\eqref{eq:score-interval} is statistically independent of the true event length
$L=\estar-\sstar$ (relative boundary error does not depend on event length). Then
under Assumptions \ref{a:nested}--\ref{a:frozen}, \cover{} is length-conditional,
$\Prob(\Igt\subseteq\Cset_{\qhat}(X)\mid L=l)\ge1-\alpha$ for every $l$. The
seconds-valued score \eqref{eq:score-interval} with $\ell(X)\equiv1$ does not
satisfy this in general, because absolute boundary error grows with event length.
\end{proposition}
% \begin{proof}[Proof sketch]
% Under the scale model $\score_{n+1}\perp L_{n+1}$, and $\qhat$ depends only on the
% calibration scores, so it too is independent of $L_{n+1}$ under i.i.d.\ sampling;
% conditioning on $L_{n+1}=l$ therefore leaves
% $\Prob(\score_{n+1}\le\qhat\mid L_{n+1}=l)=\Prob(\score_{n+1}\le\qhat)\ge1-\alpha$,
% which is conditional coverage by Assumption~\ref{a:nested}. Full proof in the Appendix.
% \end{proof}

\paragraph{Coverage when one video breaks exchangeability.}
Calibration and test moments taken from the same video are dependent, violating
Assumption \ref{a:exch}. The next result bounds the resulting loss and shows it
vanishes as the dependence does.

\begin{proposition}[Coverage under bounded non-exchangeability]\label{prop:nonexch}
Let $Z_i=(X_i,\Igt_i)$, write $Z_{1:n+1}=(Z_1,\dots,Z_{n+1})$, and let
$Z_{1:n+1}^{(i)}$ denote this sequence with the test point $Z_{n+1}$ and
calibration point $Z_i$ swapped. Under Assumptions \ref{a:nested}--\ref{a:frozen},
with the exchangeability defect
$\Delta=\tfrac{1}{n+1}\sum_{i=1}^{n} d_{\mathrm{TV}}\!\big(Z_{1:n+1},Z_{1:n+1}^{(i)}\big)$,
\cover{} satisfies
$\Prob(\Igt_{n+1}\subseteq\Cset_{\qhat}(X_{n+1}))\ge 1-\alpha-\Delta$.
\end{proposition}
% \begin{proof}[Proof sketch]
% Under exchangeability the rank of $\score_{n+1}$ among the $n+1$ scores is uniform,
% so $\Prob(A)\ge1-\alpha$ for $A=\{\Igt_{n+1}\subseteq\Cset_{\qhat}(X_{n+1})\}=
% \{\score_{n+1}\le\qhat\}$ (Proposition~\ref{prop:cov}). Each swap of the test point
% with a calibration point $i$ moves $\Prob(A)$ by at most
% $d_{\mathrm{TV}}(Z_{1:n+1},Z_{1:n+1}^{(i)})$, since total-variation distance bounds
% the change of any event probability; averaging the exact rank bound over the $n+1$
% positions and charging one such term per swap yields $\Prob(A)\ge1-\alpha-\Delta$.
% This is the beyond-exchangeability bound of \citet[Theorem~2]{barber2023beyond} with
% uniform weights, $\Delta=0$ when the moments are exchangeable, recovering
% Proposition~\ref{prop:cov}, and the full argument, the exchangeable case, and the
% behavior of $\Delta$ under mixing are in the Appendix.
% \end{proof}

%%%%%%%%%%%%%%%%%%%%%%%%%%%%%%%%%%%%%%%%%%%%%%%%%%%%%%%%%%%%%%%%%%%%%%%%%%%%%%%%
\section{Experiments}
\label{sec:experiments}

Table~\ref{tab:design} lists the six studies. Each isolates one question,
with coverage as the validity check and region length as the cost.

\begin{table}[t]
  \centering
  \caption{\textbf{Evaluation design.} Each row isolates one question answered in \cref{sec:Results} and the Appendix.}
  \label{tab:design}
  \small
  \renewcommand{\arraystretch}{1}
  \begin{tabular}{@{}>{\raggedright\arraybackslash}p{0.28\linewidth} p{0.6\linewidth}@{}}
    \toprule
    \hb
    \kilth{Study} & \kilth{Question it answers} \\
    \midrule
    Coverage validity &
    Does empirical coverage match $1-\alpha$ across the $\alpha$-grid, for each base grounder? \\
    Efficiency &
    Interval vs.\ super-level-set vs.\ fixed-margin scores at matched coverage. \\
    Risk control &
    Does RCPS miss-rate control hold at level $\alpha$ with probability $1-\delta$? \\
    Conditional coverage &
    Coverage by event length. Does Mondrian calibration restore group coverage? \\
    Cross-dataset transfer &
    Distribution shift under cross-dataset calibration, with weighted conformal as the remedy. \\
    Calibration size &
    Coverage vs.\ calibration size $n$, and wrapper compute cost. \\
    \bottomrule
  \end{tabular}
\end{table}

\paragraph{Datasets and base grounders}
We use Charades-STA \citep{gao2017tall}, ActivityNet-Captions
\citep{krishna2017anet}, and QVHighlights \citep{lei2021qvhighlights}. Each dataset
is split at the video level into a calibration pool and a test set (40/60, frozen at
seed 42), disjoint from anything used to build the grounder. Splitting by video
rather than by query keeps calibration and test exchangeable. Unless noted, coverage
variability is estimated by resampling within the frozen calibration pool (50
resamples at 50\% of its videos). To test model-agnosticism, we wrap two kinds of
grounder: a trained DETR-style localizer (QD-DETR \citep{moon2023qddetr} via
Lighthouse \citep{taichi2024emnlp}) and a black-box video--language model prompted
for interval retrieval (Qwen2.5-VL-7B-Instruct \citep{bai2025qwen25vl}). When the
black-box grounder returns no interval, a refusal policy sets $\Cset(X)=[0,T]$ for
every $\lambda$, covering trivially at maximal length. We, therefore, report Qwen metrics two ways: marginal and conditional on non-refusal.
The super-level-set score uses QVHighlights saliency where available. Lighthouse caps
usable video length, so we drop videos longer than 150 seconds. This affects only ActivityNet-Captions, removing 34.5\% of its original pool. Where QVHighlights ground truth spans multiple intervals, we score the envelope interval, i.e.\ the earliest start and latest end timestamp. Section \ref{app:extra} of the Appendix adds three further grounders, including SRAM \cite{ma2024beyond}, a training-based grounder that reports boundary uncertainty.

\paragraph{Metrics}
We report empirical coverage (the fraction of test points with
$\Igt\subseteq\Cset(X)$), efficiency (mean temporal length of $\Cset$, plus the
number of connected components for the super-level-set score), IoU between $\Cset$
and $\Igt$, and miss-rate. Coverage is the validity check and length is the cost.

\paragraph{Coverage validity}
This study underlies every other one here: it verifies that Proposition \ref{prop:cov} holds empirically, at the level a user requests. We sweep target coverage over a fixed six-point grid,
$\alpha\in\{0.05, 0.1, 0.2, 0.3, 0.4, 0.5\}$,
and calibrate and evaluate at every grid point independently, for both base grounders, on all three datasets, using both instantiations of the interval-widening score (Eq. \ref{eq:score-interval}) and, on QVHighlights, the super-level-set score (Eq. \ref{eq:score-sls}). For each combination, we report realized coverage against target, together with the mean region length at the same operating points (the coverage–efficiency frontier), and the resampling-based variability described above.

\paragraph{Efficiency at matched coverage}
Efficiency is the size of the region required to achieve the specified coverage. For a focused comparison between the different score types, reliability is fixed at 2 operating points ($1-\alpha=[0.8,0.9]$), rather than the full grid swept for coverage validity. We compare length-normalized interval score (\normscore), fixed-margin (seconds) interval score (\secscore), and the super-level-set score (\supralevel) efficiencies, at each target and for both grounders. Two further comparisons ask what the calibration step and the choice of family are worth. We sweep an uncalibrated fixed margin over six hand-picked constants through the same nested family and evaluation engine, isolating calibration as the only difference. We run a per-boundary two-parameter variant against the single-parameter score, and read its calibrated $\qhat_{\mathrm{s}}/\qhat_{\mathrm{e}}$ ratio as a diagnostic of each grounder's boundary asymmetry.

\section{Results}
\label{sec:Results}

\subsection{Coverage Validity}
\label{subsec:E0}

\begin{figure}[t]
  \centering
  \kilgraphics{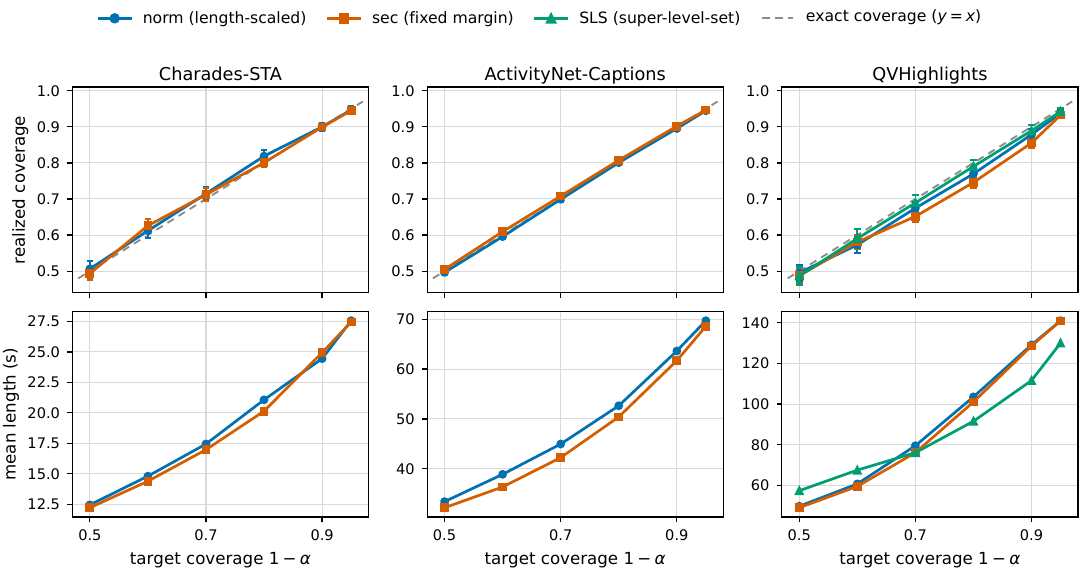}
  \caption{\textbf{overage validity and efficiency (QD-DETR).}
  Top row: realized
    coverage against the target $1-\alpha$. The dashed diagonal is exact coverage. Bottom row:
    mean region length against the same target, the coverage--efficiency frontier.
    Columns are the three datasets. The score families are \normscore\ (length-scaled),
    \secscore\ (fixed margin), and \supralevel\ (super-level-set).}
  \label{fig:coverage}
\end{figure}

\begin{figure}[t]
  \centering
\kilgraphics{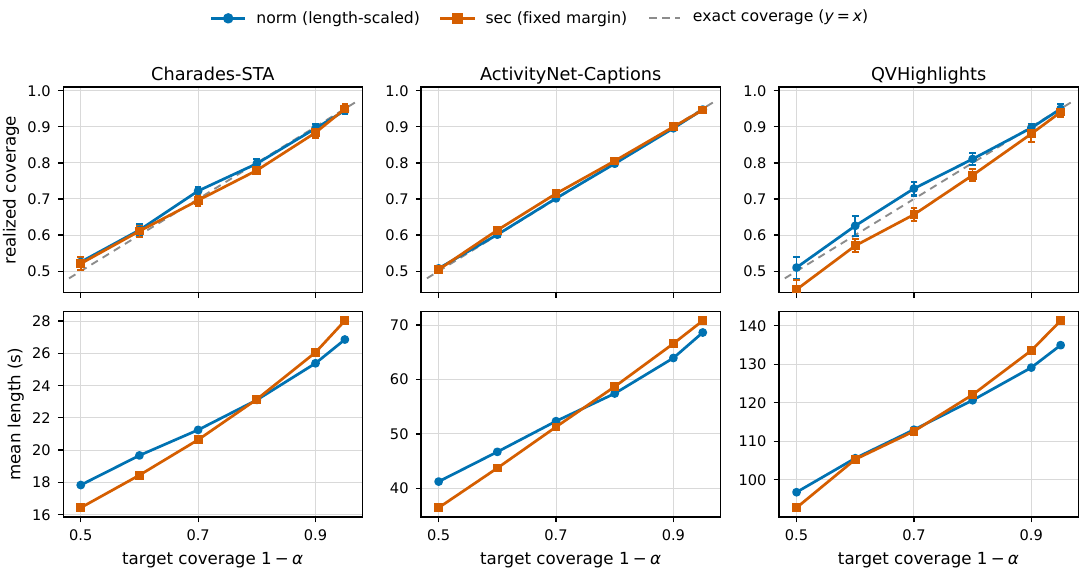}
  \caption{\textbf{Coverage validity and efficiency (Qwen2.5-VL-7B-Instruct).}
  For the two interval families the
    black-box grounder supports (it exposes no relevance signal, so there is no
    \supralevel\ curve). Marginal realized coverage tracks the target across datasets,
    confirming that the wrapper calibrates a grounder reached only through its predicted
    intervals.}
  \label{fig:E0_Qwen}
\end{figure}

Shown in Figures \ref{fig:coverage} and \ref{fig:E0_Qwen}, realized coverage tracks the target
closely over the full grid $1-\alpha \in \{0.50, 0.60, 0.70, 0.80, 0.90,
0.95\}$, confirming Proposition \ref{prop:cov} empirically for most
dataset--grounder pairs, with ActivityNet at $0.895\pm0.003$
($n_\mathrm{cal}=7811$) and Charades at $0.900\pm0.011$
($n_\mathrm{cal}=1500$) at the $0.90$ target. One case departs: QD-DETR on
QVHighlights runs below target ($0.879\pm0.016$), under the $k/(n+1)$ floor
Proposition \ref{prop:beta} in the Appendix predicts for i.i.d.\ calibration rows. A
video-level permutation test rules out the frozen split as the cause
(2000 permutations, no significant calibration-vs-test difference for either
grounder under either score family, $p\in[0.08,0.52]$; procedure in the
Appendix \ref{app:cov-further-results}). The gap is grounder-specific and traces to conditional coverage:
QD-DETR undercovers QVHighlights' Long and Multi-window strata by $0.148$ and
$0.127$ against Qwen's $0.008$ and $0.046$ (Appendix Table \ref{tab:e3}), and the
dataset's small, query-clustered pool ($625$ rows over $190$ videos) lets that
subgroup failure move the marginal average.

The super-level-set family's connected-components count averages between 1.25 and 1.51 across the grid on QVHighlights, highlighting disconnected behavior described in \cref{sec:method}. 

Qwen's refusal rate, instances where the grounder returns no usable interval, is 10.4\% on Charades-STA, 2.9\% on ActivityNet-Captions, and 2.6\% on QVHighlights. Marginal numbers are the ones the guarantee is about, but the non-refusal split is needed to compare Qwen's efficiency fairly against QD-DETR, which never refuses.

\subsection{Efficiency at Matched Coverage}
\label{subsec:E1}

\begin{table}[t]
  \centering
  \caption{\textbf{Region length at matched target coverage.}
  Region length (seconds, $\downarrow$) for both grounders.
  The best score family for each grounder, dataset, and target is shown in
  \textbf{bold}. Blank entries indicate that the score family is not applicable.
  Qwen values are marginal. Table \ref{tab:e1-qwen-refusal} in the Appendix reports non-refused
  length separately.}
  \label{tab:e1}
  \LARGE
  \resizebox{\textwidth}{!}{%
  \begin{tabular}{@{}ll
      *{3}{S[table-format=3.1]c}
      *{3}{S[table-format=3.1]c}@{}}

    \toprule
    \hb
    & &
    \multicolumn{6}{c}{$1-\alpha=0.80$ ($\downarrow$)} &
    \multicolumn{6}{c}{$1-\alpha=0.90$ ($\downarrow$)} \\

    \cmidrule(lr){3-8}
    \cmidrule(lr){9-14}
    \hb
    & &
    \multicolumn{2}{c}{\normscore} &
    \multicolumn{2}{c}{\secscore} &
    \multicolumn{2}{c}{\supralevel} &
    \multicolumn{2}{c}{\normscore} &
    \multicolumn{2}{c}{\secscore} &
    \multicolumn{2}{c}{\supralevel} \\

    \cmidrule(lr){3-4}
    \cmidrule(lr){5-6}
    \cmidrule(lr){7-8}
    \cmidrule(lr){9-10}
    \cmidrule(lr){11-12}
    \cmidrule(lr){13-14}
    \hb
    \kilth{Grounder} &
    \kilth{Dataset} &
    \kilth{Len.} & \kilth{CI} &
    \kilth{Len.} & \kilth{CI} &
    \kilth{Len.} & \kilth{CI} &
    \kilth{Len.} & \kilth{CI} &
    \kilth{Len.} & \kilth{CI} &
    \kilth{Len.} & \kilth{CI} \\

    \midrule

    \multirow{3}{*}{QD-DETR}
      & Charades-STA
      & 21.1 & {\scriptsize[20.61,21.53]}
      & \bfseries \textbf{20.1} & {\scriptsize[19.84,20.42]}
      & {--} & {--}
      & \bfseries 24.4 & {\scriptsize[23.96,24.95]}
      & 24.9 & {\scriptsize[24.59,25.30]}
      & {--} & {--} \\

      & ActivityNet-Captions
      & 52.6 & {\scriptsize[51.29,54.00]}
      & \bfseries \textbf{50.4} & {\scriptsize[49.42,51.38]}
      & {--} & {--}
      & 63.6 & {\scriptsize[62.02,65.31]}
      & \bfseries \textbf{61.7} & {\scriptsize[60.43,62.97]}
      & {--} & {--} \\

      & QVHighlights
      & 103.5 & {\scriptsize[100.95,106.19]}
      & 101.0 & {\scriptsize[99.27,102.84]}
      & \bfseries \textbf{91.5} & {\scriptsize[87.72,95.23]}
      & 129.1 & {\scriptsize[127.19,131.14]}
      & 128.6 & {\scriptsize[127.26,130.01]}
      & \bfseries \textbf{111.6} & {\scriptsize[108.22,114.94]} \\

    \addlinespace

    \multirow{3}{*}{Qwen}
      & Charades-STA
      & \bfseries \textbf{23.1} & {\scriptsize[22.47,23.78]}
      & \bfseries \textbf{23.1} & {\scriptsize[22.71,23.61]}
      & {--} & {--}
      & \bfseries 25.4 & {\scriptsize[24.78,25.99]}
      & 26.1 & {\scriptsize[25.59,26.53]}
      & {--} & {--} \\

      & ActivityNet-Captions
      & \bfseries \textbf{57.4} & {\scriptsize[55.90,58.87]}
      & 58.7 & {\scriptsize[57.42,59.83]}
      & {--} & {--}
      & \bfseries 63.9 & {\scriptsize[62.46,65.54]}
      & 66.6 & {\scriptsize[65.14,68.01]}
      & {--} & {--} \\

      & QVHighlights
      & \bfseries \textbf{120.6} & {\scriptsize[117.81,123.25]}
      & 122.1 & {\scriptsize[120.56,123.51]}
      & {--} & {--}
      & \bfseries 129.1 & {\scriptsize[127.00,131.30]}
      & 133.5 & {\scriptsize[132.42,134.66]}
      & {--} & {--} \\

    \bottomrule
  \end{tabular}}
\end{table}

Table~\ref{tab:e1} reports mean region length at $1-\alpha \in \{0.80, 0.90\}$ for each score family, per grounder. Figure \ref{fig:e1-frontier} in the Appendix plots the full QD-DETR frontier from which these two operating points are drawn. No family dominates uniformly. For QD-DETR on Charades-STA, the two interval families cross between the two targets: \secscore\ is shorter at 0.80 (20.1s vs.\ 21.1s) while \normscore\ is shorter at 0.90 (24.4s vs.\ 24.9s). On ActivityNet-Captions, \secscore\ is shorter at both targets with no crossing. On QVHighlights, \supralevel\ is shortest at both tested targets (91.5s and 111.6s vs.\ 101--103s and 128--129s for the interval families). This should not be concluded as a general property of the super-level-set family, since the fuller $\alpha$-grid in Figure \ref{fig:coverage} shows it losing to both interval families at looser targets ($1-\alpha \le 0.70$) and only overtaking them from roughly 0.80 upward.

Qwen shows a different, more consistent pattern: \normscore\ is shortest or tied on every dataset at both targets. Table \ref{tab:e1-qwen-refusal} in the Appendix isolates genuine calibrated widening from the refusal-driven inflation: Qwen's marginal length is measurably larger than its non-refused length wherever the refusal rate is non-trivial (e.g.\ Charades-STA: 25.4s marginal vs.\ 24.5s non-refused for \normscore), so refusal-conditional is the fair basis for comparing efficiency against QD-DETR. Looking at the full sweep in Figure \ref{fig:E0_Qwen}, \secscore\ is shorter at looser targets, with \normscore\ overtaking at $1-\alpha=0.8$ onwards.

\begin{table}[t]
  \centering
  \caption{\textbf{Coverage under fixed uncalibrated margins.}
  Realized coverage for the uncalibrated margin score (\secscore) swept across
  six hand-picked constants for both grounders. The final column reports the
  range (maximum minus minimum coverage) across the sweep for each row.
  \colorbox{underc}{Shaded} cells fall below $0.80$, the looser of COVER's two
  headline targets. For comparison, COVER's calibrated $\hat{\lambda}$ realizes
  coverage within approximately $0.02$ of the requested target on every
  (grounder, dataset) pair (Table~\ref{tab:e1},
  Figures~\ref{fig:coverage} and~\ref{fig:E0_Qwen}).}
  \label{tab:e1-uncalibrated}
  \scriptsize
  \setlength{\tabcolsep}{4.5pt}

  \resizebox{\linewidth}{!}{%
  \begin{tabular}{@{}ll*{6}{r}r@{}}
    \toprule
    \hb
    \kilth{Grounder} &
    \kilth{Dataset} &
    \kilth{$\lambda=5$} &
    \kilth{$\lambda=10$} &
    \kilth{$\lambda=20$} &
    \kilth{$\lambda=40$} &
    \kilth{$\lambda=80$} &
    \kilth{$\lambda=160$} &
    \kilth{Range} \\
    \midrule

    QD-DETR &
    Charades-STA &
    \colorbox{underc}{0.662} &
    0.819 &
    0.958 &
    1.000 &
    1.000 &
    1.000 &
    0.338 \\

    QD-DETR &
    ActivityNet-Captions &
    \colorbox{underc}{0.536} &
    \colorbox{underc}{0.644} &
    \colorbox{underc}{0.766} &
    0.883 &
    0.969 &
    1.000 &
    0.464 \\

    QD-DETR &
    QVHighlights &
    \colorbox{underc}{0.474} &
    \colorbox{underc}{0.552} &
    \colorbox{underc}{0.629} &
    \colorbox{underc}{0.723} &
    0.875 &
    1.000 &
    0.526 \\

    \addlinespace

    Qwen &
    Charades-STA &
    \colorbox{underc}{0.449} &
    \colorbox{underc}{0.659} &
    0.912 &
    0.999 &
    1.000 &
    1.000 &
    0.551 \\

    Qwen &
    ActivityNet-Captions &
    \colorbox{underc}{0.335} &
    \colorbox{underc}{0.424} &
    \colorbox{underc}{0.566} &
    \colorbox{underc}{0.759} &
    0.950 &
    1.000 &
    0.665 \\

    Qwen &
    QVHighlights &
    \colorbox{underc}{0.096} &
    \colorbox{underc}{0.104} &
    \colorbox{underc}{0.169} &
    \colorbox{underc}{0.325} &
    \colorbox{underc}{0.722} &
    1.000 &
    0.904 \\

    \bottomrule
  \end{tabular}}
\end{table}

Table~\ref{tab:e1-uncalibrated} presents a different approach for widening: choosing $\lambda$ by hand instead of calibration. A fixed expansion margin cannot reliably achieve a desired coverage target across datasets or models, with realized coverage ranging from $0.096$ to $1.000$ across the six constants tested. Calibration is necessary to obtain the desired coverage efficiently. The full analysis is in the Appendix.

\paragraph{One parameter or two is a property of the grounder}
Tables \ref{tab:e1-cqr-excess} and \ref{tab:e1-cqr-excess} in the Appendix compare the
single-parameter score against the per-boundary two-parameter variant \cref{fig:method}, and the answer splits by grounder. For QD-DETR the
single parameter is the right choice: the two-parameter variant is longer in 9
of 12 combinations, and its separately calibrated
boundaries land at comparable values throughout. For Qwen the ordering inverts: two parameters are shorter in 9 of 12 combinations, by $-3.4\%$ on average and up to $-9.1\%$ on QVHighlights at
$1-\alpha=0.80$. Both members carry the same coverage guarantee, so this is a
choice about efficiency alone, and the calibrated ratio in
Table \ref{tab:e1-cqr-lambdas} in the Appendix tells a practitioner which member to use before
any test data is seen.

\paragraph{Calibration as measurement: Qwen knows when events start, not when
they end} The reason two parameters help on Qwen is a pronounced asymmetry in
the grounder itself, which the calibration step measures directly.
$\hat\lambda_{\text{start}}$ calibrates to zero or near zero in seven of Qwen's twelve
rows and never exceeds $5.44$, while $\hat\lambda_{\text{end}}$ is never below
$3.31$ and reaches $115.4$ (Appendix Table \ref{tab:e1-cqr-lambdas}). The model locates when an
event begins reliably, and essentially all of its boundary error is in not
knowing when the event ends. It explains the efficiency result above: the single-parameter score takes the maximum over sides, so an
end-driven $\hat\lambda$ is applied to an already-adequate start side, and the
resulting waste scales with absolute event duration, which is why the saving is
largest on QVHighlights and smallest on Charades-STA.

\subsection{Further Studies}
\label{subsec:E2-E5}
Four studies stress the guarantee where it is most likely to break, and each
finds the predicted behavior together with the boundary at which it stops
holding. \emph{Risk control} verifies that the miss rate remains within the prescribed risk budget ($\delta=0.1$) across all configurations, with interval length increasing by only $1.4$--$14.8\%$ and shrinking at the expected $O(n^{-1/2})$ rate as the calibration set grows (Appendix Table \ref{tab:e2}). \emph{Conditional coverage} shows that Mondrian calibration substantially improves coverage within strata, reducing the worst-stratum deviation by more than $2.5\times$ on QD-DETR, but its benefit diminishes when the calibration pool is too small to populate each stratum (Appendix Table \ref{tab:e3} and Table \ref{tab:e3-sec}). \emph{Cross-dataset transfer} demonstrates that duration-aware reweighting largely restores coverage under shift, reducing the coverage gap from $18.8$ percentage points to $3.6$, while also illustrating that no reweighting method can compensate for regions of the target distribution unsupported by the source data, a positivity limitation (Appendix Table \ref{tab:e4a} and Table \ref{tab:e4b}). Finally, the \emph{calibration-size} study confirms the finite-sample behavior predicted by Theorem \ref{thm:eff}: coverage remains valid as the calibration size increases, with uncertainty shrinking from $\pm0.042$ to $\pm0.007$ at the expected $O(n^{-1/2})$ rate, while still providing usable guarantees with as few as 20 calibration videos, albeit with wider uncertainty (Appendix Table \ref{tab:e5}). Full protocols, tables, and discussion are presented in
Section \ref{app:extra} of the Appendix, along with evaluation on additional grounders (Appendix Table \ref{tab:tr-detr} and \ref{tab:internvl}) and a downstream-utility study connecting the certified
width to a deployment rule (Appendix Table \ref{tab:e-utility-selective} and \ref{tab:e-utility-corr}).

\section{Discussion and Limitations}
\label{sec:limits}

\paragraph{Exchangeability and shift}
The coverage guarantee rests on Assumption \ref{a:exch}, and our experiments locate a hard bound between the two ways it can fail. Across datasets the marginal distribution shifts, and weighted conformal
prediction \cite{tibshirani2019covariate} restores coverage under known covariate
shift, but only where the calibration data has support: the cross-dataset transfer study finds reweighting repairs one direction fully
and the other not at all, and the failing direction is a positivity limit no
reweighting scheme can fix. Within a long video, clips and queries are temporally
dependent, the beyond-exchangeability literature
\citep{barber2023beyond,gibbs2021aci,farinhas2023nonexch} gives adaptive remedies for when calibration and test share a video, and
Proposition \ref{prop:nonexch} bounds the resulting loss. We do not claim
coverage under arbitrary shift.

\paragraph{Marginal, not conditional}
\cover{} guarantees marginal coverage, and the conditional-coverage study shows what that permits: QD-DETR's
QVHighlights coverage is at $0.879$ marginally while its Long stratum sits at
$0.752$. Stratified calibration (Proposition \ref{prop:mondrian} in the Appendix) and score
normalization (Proposition \ref{prop:lcc}) each recover part of this, but exact
attribute-free conditional coverage remains unattainable distribution-free
\citep{barber2021limits}. A user who cares about a specific subpopulation should
calibrate on it.

\paragraph{Efficiency and calibration data}
Coverage is vacuous without efficiency, and the wrapper inherits the grounder's
quality rather than making a weak grounder strong. The super-level-set score is
only as good as the relevance signal it thresholds. Where partial overlap suffices, the miss-rate or an
IoU-based risk applies unchanged. The wrapper needs a labeled, exchangeable
calibration set of size $n$, which sets the finite-sample slack $1/(n+1)$. The
calibration-size study puts the practical floor at roughly $20$ videos.

% ============================== CONCLUSION =================================
\section{Conclusion}
\label{sec:conclusion}

Event boundaries in continuous video are ambiguous, so the ground truth for
temporal grounding is a distribution over intervals and a single predicted
interval is the wrong object to report. We presented \cover{}, a post-hoc,
model-agnostic wrapper that changes the output: it calibrates a temporal
nonconformity score on a small labeled set and widens any grounder's prediction
so the true moment is contained with a user-chosen probability, with no
retraining and no white-box access. We prove the region is asymptotically the
tightest of its family, that coverage holds within event-length strata under a
scale model, and that it degrades gracefully when one video breaks
exchangeability.

What the guarantee costs is measured, and what it reveals is not only a
guarantee. Across three benchmarks and five grounders, hand-picked margins
swing realized coverage by up to $0.904$ and an evidential grounder built for
uncertainty quantification misses its nominal level by $14$ points, while a
single calibrated scalar hits whichever target the user names.

\section*{Ethics and LLM-use Disclosure}
Large language models were used to assist with writing and editing this paper. The authors reviewed all content and take full responsibility for the work. The study relies only on publicly available datasets and models and raises no additional ethical concerns.

%%%%%%%%%%%%%%%%%%%%%%%%%%%%%%%%%%%%%%%%%%%%%%%%%%%%%%%%%%%%%%%%%%%%%%%%%%%%%%%%
%% BACK MATTER
%%%%%%%%%%%%%%%%%%%%%%%%%%%%%%%%%%%%%%%%%%%%%%%%%%%%%%%%%%%%%%%%%%%%%%%%%%%%%%%%

% \begin{contributions}
% C.P.\ and A.R.\ contributed equally: C.P.\ designed the angular channel and ran the
% QM9 and MD17 experiments; A.R.\ built the training infrastructure and the OC20
% pipeline. M.-L.C.\ contributed the theoretical analysis. S.A.-N.\ curated the defect
% benchmark. H.K.\ supervised the project and wrote the paper with input from all
% authors.
% \end{contributions}

\begin{acknowledgments}
We thank the members of the Kurban Intelligence Lab for discussions, and the Electrical and Computer Engineering Department at Texas A\&M University at Qatar for the provided support. 
\end{acknowledgments}

\begin{funding}
This work was supported by the Student Research Experience Grant (SREG) and the Department of Electrical and Computer Engineering at Texas A\&M University at Qatar.
\end{funding}

\begin{availability}
% This section SURVIVES blind mode (most venues require it), so the URL is yours
% to anonymize -- \ifbool{KIL@blind}{...}{...} does it for you, every time.
Code:
\ifbool{KIL@blind}
  {\url{https://anonymous.4open.science/r/XXXXXX}}
  {\url{https://github.com/KurbanIntelligenceLab/cover}}.
Datasets are public and cited in \cref{sec:experiments}.
\end{availability}

\begin{conflicts}
The authors declare no competing interests.
\end{conflicts}

%%%%%%%%%%%%%%%%%%%%%%%%%%%%%%%%%%%%%%%%%%%%%%%%%%%%%%%%%%%%%%%%%%%%%%%%%%%%%%%%
\bibliography{references}

%%%%%%%%%%%%%%%%%%%%%%%%%%%%%%%%%%%%%%%%%%%%%%%%%%%%%%%%%%%%%%%%%%%%%%%%%%%%%%%%

\clearpage
\appendix

\section{Standard Guarantees}
\label{app:standard}

These results are split-conformal and risk-control facts
\citep{vovk2005algorithmic,lei2018distribution,bates2021rcps} instantiated with
the temporal score. We state and prove them for completeness. Propositions
\ref{prop:tight} and \ref{prop:rcps} are referenced from \Cref{sec:theory}.

\begin{proof}[Proof of Proposition \ref{prop:cov}]
By Assumption \ref{a:nested}, $\Igt_{n+1}\subseteq\Cset_{\qhat}(X_{n+1})$ iff
$\score_{n+1}\le\qhat$, where $\score_{n+1}=\score(X_{n+1},\Igt_{n+1})$. By
Assumption \ref{a:frozen} the map $\score$ does not depend on the calibration
labels, so by Assumption \ref{a:exch} the scores $\score_1,\dots,\score_{n+1}$ are
exchangeable. For exchangeable scalars the rank of $\score_{n+1}$ among the $n+1$
values is uniform on $\{1,\dots,n+1\}$ up to ties, hence
\[
\Prob\big(\score_{n+1}\le \score_{(\lceil (n+1)(1-\alpha)\rceil)}\big)\ge
\frac{\lceil (n+1)(1-\alpha)\rceil}{n+1}\ge 1-\alpha,
\]
with $\score_{(k)}$ the $k$-th smallest calibration score and
$\qhat=\score_{(\lceil (n+1)(1-\alpha)\rceil)}$. Combining with the containment
equivalence gives the claim.
\end{proof}

\begin{proposition}[Upper bound]\label{prop:tight}
If $\score_1,\dots,\score_{n+1}$ are almost surely distinct, the coverage in
Proposition \ref{prop:cov} is at most $1-\alpha+\tfrac{1}{n+1}$.
\end{proposition}
\begin{proof}
Without ties the rank of $\score_{n+1}$ is exactly uniform, so
$\Prob(\score_{n+1}\le\qhat)=\lceil (n+1)(1-\alpha)\rceil/(n+1)$, and
$\lceil (n+1)(1-\alpha)\rceil<(n+1)(1-\alpha)+1$ gives the bound.
\end{proof}

\begin{proposition}[Miss-rate control]\label{prop:rcps}
Under Assumption \ref{a:exch}--Assumption \ref{a:frozen}, with the monotone loss
$\mathcal{L}_\lambda(X,\Igt)=\mathbf{1}[\score(X,\Igt)>\lambda]$, the risk-controlling choice
$\hat\lambda$ of \citet{bates2021rcps} from the calibration scores satisfies
$\Prob(\E[\mathcal{L}_{\hat\lambda}(X_{n+1},\Igt_{n+1})]\le\alpha)\ge 1-\delta$, the outer
probability over the calibration draw.
\end{proposition}
\begin{proof}
By Assumption \ref{a:nested}, $\mathcal{L}_\lambda=\mathbf{1}[\score>\lambda]$ is
nonincreasing in $\lambda$, so the risk $R(\lambda)=\E[\mathcal{L}_\lambda]$ is nonincreasing.
This is the monotone-loss setting of \citet{bates2021rcps}: form a pointwise
upper-confidence bound $\hat R^+(\lambda)$ on $R(\lambda)$ from the calibration
scores via a concentration inequality \citep{boucheron2013concentration} and set
$\hat\lambda=\inf\{\lambda:\hat R^+(\lambda')\le\alpha\ \forall\lambda'\ge\lambda\}$.
The RCPS theorem of \citet{bates2021rcps} gives $\Prob(R(\hat\lambda)\le\alpha)\ge1-\delta$;
$R(\hat\lambda)=\E[\mathcal{L}_{\hat\lambda}]$ yields the statement.
\end{proof}

\begin{proposition}[Group-conditional coverage]\label{prop:mondrian}
Partition inputs into groups by a discrete attribute (event-length bucket or
class). If a separate threshold $\qhat^{(k)}$ is calibrated on the $n_k$
calibration points of group $k$, then for a test point in group $k$,
$\Prob(\Igt\subseteq\Cset_{\qhat^{(k)}}(X)\mid \text{group}=k)\ge 1-\alpha$, with
finite-sample slack $1/(n_k+1)$.
\end{proposition}
\begin{proof}
Within group $k$ the points are exchangeable, so Proposition \ref{prop:cov} applies
to the group's scores with $n_k$ calibration points, giving coverage at least
$1-\alpha$ conditional on the group and the tightness $1/(n_k+1)$ from
Proposition \ref{prop:tight}.
\end{proof}

\begin{proposition}[Distribution of coverage given the calibration set]\label{prop:beta}
Suppose the scores $\score_1,\dots,\score_{n+1}$ are i.i.d.\ with a continuous
distribution, and let $k=\lceil(n+1)(1-\alpha)\rceil$ so that
$\qhat=\score_{(k)}$. Conditional on the calibration scores, the realized coverage
$C=\Prob(\score_{n+1}\le\qhat\mid \score_1,\dots,\score_n)$ follows a
$\mathrm{Beta}(k,\,n+1-k)$ law. Hence $\E[C]=k/(n+1)\ge1-\alpha$ and
$\mathrm{sd}(C)=\big[k(n+1-k)/((n+1)^2(n+2))\big]^{1/2}=O(n^{-1/2})$.
\end{proposition}
\begin{proof}
Let $U_i=F_S(\score_i)$. By the probability integral transform the $U_i$ are
i.i.d.\ $\mathrm{Uniform}(0,1)$, and since $\score_{n+1}$ is a fresh draw,
$C=F_S(\score_{(k)})=U_{(k)}$, the $k$-th order statistic of $n$ i.i.d.\ uniform
variables, which has the $\mathrm{Beta}(k,n+1-k)$ distribution. Its mean is
$k/(n+1)\ge1-\alpha$ and its variance is $k(n+1-k)/\big((n+1)^2(n+2)\big)=O(1/n)$,
so the standard deviation is $O(n^{-1/2})$.
\end{proof}

\noindent Proposition \ref{prop:beta} is the finite-sample counterpart of
Theorem \ref{thm:eff}(ii): it is why the resampled coverage in the Results section concentrates near the target with a spread that narrows
at the $n^{-1/2}$ rate as the calibration pool grows.

\section{Proofs of the Grounding-Specific Results}
\label{app:proofs}

\begin{proof}[Proof of Theorem \ref{thm:eff}]
\emph{(i) Length identity.} The candidate region is
$\Cset_\lambda(X)=[\shat-\lambda\ell(X),\,\ehat+\lambda\ell(X)]\cap[0,T]$. The
bracket has length $(\ehat+\lambda\ell)-(\shat-\lambda\ell)=(\ehat-\shat)+2\lambda\ell
=\ell(1+2\lambda)$ using $\ell=\ehat-\shat$, and intersecting with $[0,T]$ can only
decrease length, so $|\Cset_\lambda(X)|=\min\{\ell(1+2\lambda),T\}$. Setting
$\lambda=\qhat$ and taking expectations,
$\E|\Cset_{\qhat}(X)|\le\E[\ell(X)(1+2\qhat)]=\E[\ell(X)]+2\,\E[\ell(X)\qhat]$. If the
test input $X_{n+1}$ is independent of the calibration data then $\ell(X_{n+1})\perp\qhat$,
since $\qhat$ is a function of the calibration scores, so $\E[\ell(X)\qhat]=\E[\ell(X)]\E[\qhat]$
and the bound equals $\E[\ell(X)](1+2\,\E[\qhat])$.

\emph{(ii) Consistency.} Assume the calibration scores are i.i.d.\ (a strengthening
of Assumption \ref{a:exch}, as in the theorem statement); the Glivenko--Cantelli
and central-limit steps below use it. By Assumption \ref{a:nested},
$\Igt\subseteq\Cset_\lambda(X)\iff\score(X,\Igt)\le\lambda$, so the family coverage at
level $\lambda$ is $\Prob(\score\le\lambda)=F_S(\lambda)$. The empirical distribution
$\hat F_n$ of the calibration scores converges uniformly to $F_S$
(Glivenko--Cantelli), and $\qhat$ is the $k$-th order statistic with
$k=\lceil(n+1)(1-\alpha)\rceil$, i.e.\ the empirical $(1-\alpha+o(1))$-quantile, so
$\qhat\to q^\star:=Q_{1-\alpha}(F_S)$ almost surely. If $F_S$ is differentiable at
$q^\star$ with $f_S(q^\star)>0$, the Bahadur representation of sample quantiles gives
$\qhat=q^\star+\big[(1-\alpha)-\hat F_n(q^\star)\big]/f_S(q^\star)+o_{\mathrm p}(n^{-1/2})$;
since $\sqrt{n}\big[(1-\alpha)-\hat F_n(q^\star)\big]$ is asymptotically normal by the
central limit theorem, $\sqrt{n}\,(\qhat-q^\star)=O_{\mathrm p}(1)$.

\emph{(iii) Family length-optimality.} For continuous $F_S$, coverage $F_S(\lambda)$ is
nondecreasing and length $\ell(1+2\lambda)$ is strictly increasing in $\lambda$. Any
$\lambda$ with $F_S(\lambda)\ge1-\alpha$ satisfies $\lambda\ge q^\star$ by definition of
the quantile; every such $\lambda>q^\star$ gives a strictly longer region, and every
$\lambda<q^\star$ gives $F_S(\lambda)<1-\alpha$. Hence $\Cset_{q^\star}$ is the unique
shortest region in the family meeting the coverage constraint.
\end{proof}

\begin{proof}[Proof of Theorem \ref{thm:eff}(iv)]
For any $\tau$, $\Igt\subseteq\{t:f_X(t)\ge\tau\}$ holds iff $f_X(t)\ge\tau$ for every
$t\in\Igt$, i.e.\ $\min_{t\in\Igt}f_X(t)\ge\tau$, i.e.\
$\score(X,\Igt)=-\min_{t\in\Igt}f_X(t)\le-\tau$. The family coverage at threshold
$\tau$ is therefore $\Prob(\score\le-\tau)=F_S(-\tau)$, nonincreasing in $\tau$. The
Lebesgue measure $\mu(\tau)=|\{t\in[0,T]:f_X(t)\ge\tau\}|$ is nonincreasing in $\tau$
because the super-level sets are nested decreasing in $\tau$. Coverage at least
$1-\alpha$ requires $F_S(-\tau)\ge1-\alpha$, i.e.\ $-\tau\ge Q_{1-\alpha}(F_S)$, i.e.\
$\tau\le-Q_{1-\alpha}(F_S)$; among these thresholds the largest,
$\tau=-Q_{1-\alpha}(F_S)$, minimizes $\E[\mu(\tau)]$ by monotonicity of $\mu$. \cover{}
uses $\tau=-\qhat$ with $\qhat\to Q_{1-\alpha}(F_S)$ by Theorem \ref{thm:eff}(ii).
\end{proof}

\begin{proof}[Proof of Proposition \ref{prop:lcc}]
Under i.i.d.\ sampling the calibration pairs $(X_i,\Igt_i)_{i\le n}$ are independent of
the test pair $(X_{n+1},\Igt_{n+1})$, so $\qhat=g(\score_1,\dots,\score_n)$ is
independent of $(\score_{n+1},L_{n+1})$. Fix $l$. Because $\qhat\perp(\score_{n+1},L_{n+1})$,
the conditional law of $\qhat$ given $L_{n+1}=l$ equals its marginal, and $\qhat$ is
conditionally independent of $\score_{n+1}$ given $L_{n+1}=l$; because the scale model
gives $\score_{n+1}\perp L_{n+1}$, the conditional law of $\score_{n+1}$ given $L_{n+1}=l$
equals its marginal. Therefore
\[
\begin{aligned}
\Prob(\score_{n+1}\le\qhat\mid L_{n+1}=l)
= \\
\int \Prob(\score_{n+1}\le t)\,dP_{\qhat}(t)
= \Prob(\score_{n+1}\le\qhat)\ge1-\alpha,
\end{aligned}
\]
the inequality by Proposition \ref{prop:cov}. By Assumption \ref{a:nested} the event
$\{\score_{n+1}\le\qhat\}$ is $\{\Igt_{n+1}\subseteq\Cset_{\qhat}(X_{n+1})\}$, giving
length-conditional coverage. For the seconds score $\ell\equiv1$,
$\score(X,\Igt)=\max(\shat-\sstar,\,\estar-\ehat)$ carries absolute units and is in
general dependent on $L$ (longer events admit larger absolute boundary errors), so
$\score_{n+1}\not\perp L_{n+1}$ and conditioning on $L_{n+1}=l$ shifts its law.
\end{proof}

\section{Proof of Proposition \ref{prop:nonexch}}
\label{app:novel}

\begin{proof}
Write the containment event as $A=\{\score_{n+1}\le\qhat\}$, which by
Assumption \ref{a:nested} equals $\{\Igt_{n+1}\subseteq\Cset_{\qhat}(X_{n+1})\}$; and
let $Z_{1:n+1}^{(i)}$ swap the test point $Z_{n+1}$ with calibration point $Z_i$.
Under exchangeability of $Z_{1:n+1}$ the rank of $\score_{n+1}$ among
$\score_1,\dots,\score_{n+1}$ is uniform and $\Prob(A)\ge1-\alpha$
(Proposition \ref{prop:cov}). For a general joint law $P$, compare $P$ with each
swapped law $P^{(i)}$, the law of $Z_{1:n+1}^{(i)}$: for the fixed event $A$, the
definition of total-variation distance gives
$|\Prob_P(A)-\Prob_{P^{(i)}}(A)|\le d_{\mathrm{TV}}(Z_{1:n+1},Z_{1:n+1}^{(i)})$. Under
$P^{(i)}$ the test point occupies calibration position $i$, so the exchangeable rank
bound applies to each of the $n+1$ positions; aggregating these one-swap bounds with
uniform weights $1/(n+1)$, as in \citet[Theorem~2]{barber2023beyond}, gives
\[
\begin{aligned}
\Prob(A)\ \ge\ 1-\alpha-\frac{1}{n+1}\sum_{i=1}^{n}
d_{\mathrm{TV}}\big(Z_{1:n+1},Z_{1:n+1}^{(i)}\big)\ \\
= 
1-\alpha-\Delta .
\end{aligned}
\]

\emph{Exchangeable case.} If $Z_{1:n+1}$ is exchangeable then for every $i$ the
transposed sequence $Z_{1:n+1}^{(i)}$ has the same law as $Z_{1:n+1}$, so each
total-variation term vanishes, $\Delta=0$, and the bound reduces to
Proposition \ref{prop:cov}. This holds in particular when the moments are i.i.d.\ (for
example one moment per video with independent videos). More generally, each term
$d_{\mathrm{TV}}(Z_{1:n+1},Z_{1:n+1}^{(i)})$ is exactly the failure of the joint law to
be invariant under exchanging the test moment with calibration moment $i$, which is
non-zero only through statistical dependence linking the test moment to the rest; for
data drawn independently across videos this dependence is confined to calibration
moments sharing the test moment's video.

\emph{Mixing.} If that within-video sequence is stationary and mixing, its dependence
decays with temporal separation, so one expects each same-video term, and hence
$\Delta$, to shrink as the test moment is separated in time from the same-video
calibration moments. A quantitative bound requires the mixing coefficients of a
specified model and is left open. We do not claim a rate.
\end{proof}

\section{Assumptions at a Glance}
\label{app:assume}
Table ~\ref{tab:assumptions} summarizes the main assumptions made for the analysis.

\begin{table}[t]
  \centering
  \caption{\textbf{Key assumptions.}
  The three assumptions underlying COVER, the role each plays in the analysis,
  and the corresponding failure mode if the assumption is violated.}
  \label{tab:assumptions}
  \small

  \resizebox{\columnwidth}{!}{%
  \begin{tabular}{@{}p{0.25\linewidth}p{0.30\linewidth}p{0.30\linewidth}@{}}
    \toprule
    \kilth{Assumption} &
    \kilth{Role in the analysis} &
    \kilth{Failure mode} \\
    \midrule

    Assumption 1: Exchangeability &
    Makes calibration scores predictive of the test score; the source of the
    coverage bound. &
    Distribution shift or within-video dependence, quantified by
    Proposition \ref{prop:nonexch}. \\

    \addlinespace

    Assumption 2: Nested family &
    Makes the score a valid threshold and the miss-loss monotone. &
    A non-nested region family breaks both coverage and RCPS. \\

    \addlinespace

    Assumption 3: Frozen pipeline &
    Keeps the scores exchangeable (no peeking at calibration labels). &
    Tuning $g$ or $\score$ on calibration data invalidates the guarantee. \\

    \bottomrule
  \end{tabular}}
\end{table}

\section{Method Schematic}
\label{app:method-fig}
Figure~\ref{fig:method} illustrates the two-phase procedure specified by ~Eq.(\ref{eq:score-interval})--Eq. (\ref{eq:qhat}) and Algorithm~\ref{alg:cover}. Figure \ref{fig:qualitative_example} showcases a qualitative example.

\begin{algorithm}[t]
\caption{\cover{} (split-conformal temporal coverage)}
\label{alg:cover}
\begin{algorithmic}[1]
\Require grounder $g$; calibration $\{(X_i,\Igt_i)\}_{i=1}^n$; level $\alpha$;
nested family $\Cset_\lambda$ and score $\score$
\For{$i=1$ to $n$} \State $\score_i \gets \score(X_i,\Igt_i)$ \EndFor
\State $\qhat \gets \lceil (n+1)(1-\alpha)\rceil$-th smallest of $\{\score_i\}$
\State \textbf{at test input} $X$: \Return $\Cset_{\qhat}(X)$
\end{algorithmic}
\end{algorithm}

\begin{figure}[t]
  \centering
\kilgraphics{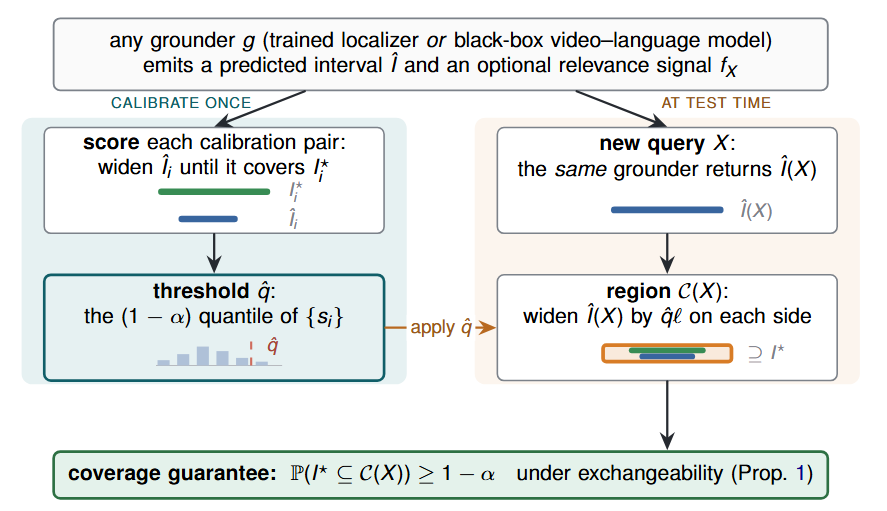}
  \caption{\textbf{The \cover{} wrapper.}
  Fhe same grounder $g$ drives both phases and
is never retrained. \emph{Calibrate once} (left, on held-out set): each example
is scored by the smallest per-side widening that makes the predicted interval
contain the true moment~Eq.(\ref{eq:score-interval}), and the threshold $\qhat$ is set
to the $(1-\alpha)$ quantile of those scores~Eq.~(\ref{eq:qhat}). \emph{At test time}
(right, per query): the predicted interval is widened by $\qhat\,\ell$ on each side,
where $\ell$ is the length scale of the score family, giving a region that contains
the true moment with probability at least $1-\alpha$ under
exchangeability (Proposition \ref{prop:cov}). A risk-control variant instead sizes the
widening to bound the expected miss-rate at a user level~(Proposition \ref{prop:rcps}).}
  \label{fig:method}
\end{figure}

\begin{figure}[t]
  \centering
\kilgraphics{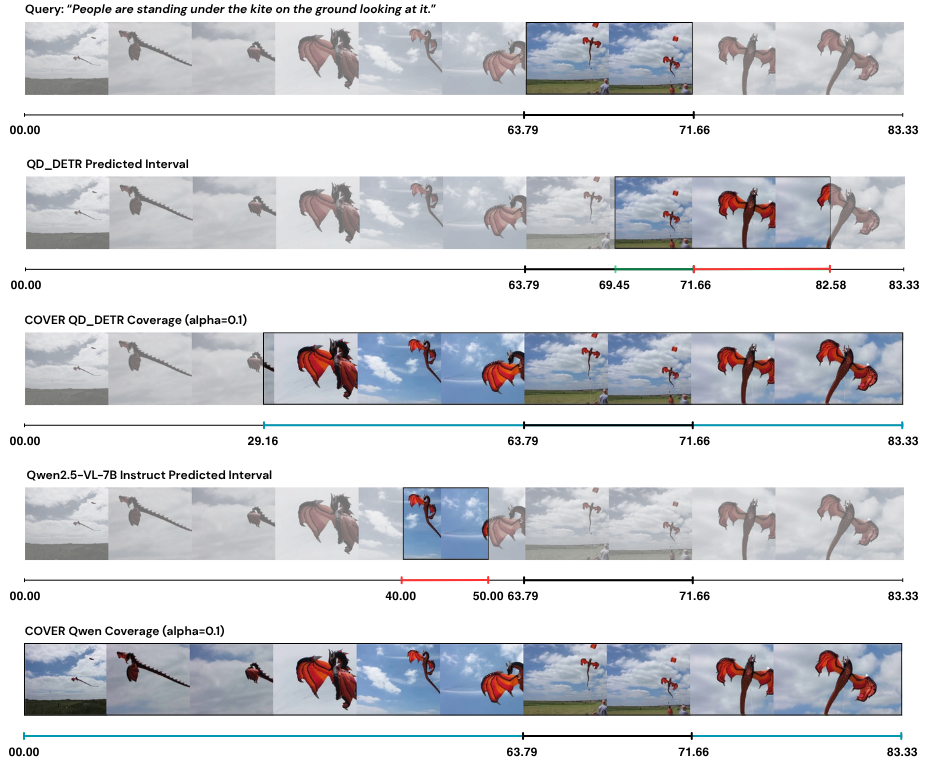}
  \caption{\textbf{Qualitative example of \cover{}}
   Cover widens the grounding interval based on the \normscore\ score of both grounders. \cover{} does not make a bad grounder good, it delivers an interval where the true moment exists with probability of $1-\alpha$. Here, $\alpha=0.1$. \textit{Timeline proportions are not uniform for visual illustration purposes.}}
  \label{fig:qualitative_example}
\end{figure}

\section{Coverage and Efficiency: Further Results}
\label{app:cov-further-results}
This section includes tables referred to in the main body of the paper, in addition to providing additional insight on the coverage-validity and efficiency findings.

\begin{figure}[t]
    \centering
    \kilgraphics{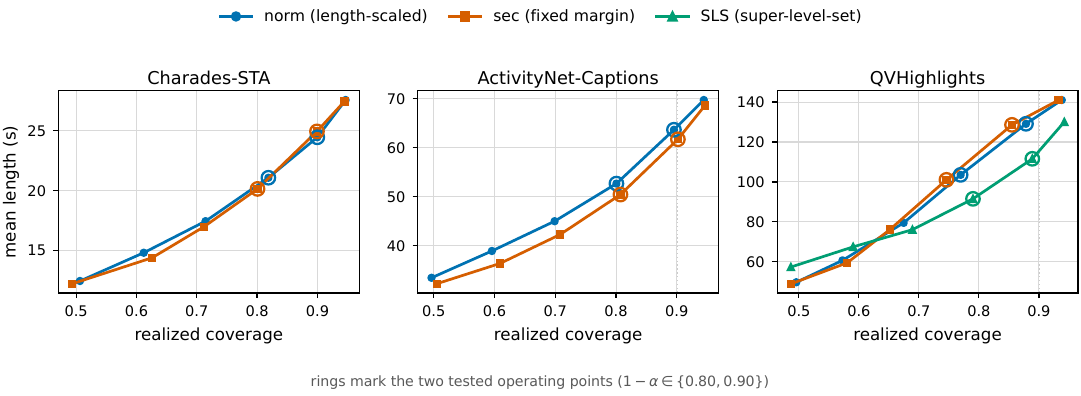}
    \caption{\textbf{Coverage-efficiency frontier, QD-DETR.} Mean region length against
    realized coverage for each score family and dataset. Rings mark the two tested
    operating points ($1-\alpha\in\{0.80,0.90\}$) reported in Table \ref{tab:e1}. The
    super-level-set family overtakes the interval families on QVHighlights at
    tighter targets.}
    \label{fig:e1-frontier}
\end{figure}

\begin{table}[t]
  \centering
  \caption{\textbf{Relative length of two-parameter calibration.}
  Percentage change in region length,
  $100\times(\mathrm{len}_{\text{2-param}}-\mathrm{len}_{\text{1-param}})/\mathrm{len}_{\text{1-param}}$,
  for both score families and both target coverages. Positive values indicate
  that the second parameter increases region length, while negative values
  indicate improved efficiency. QD-DETR generally favors a single parameter,
  whereas Qwen benefits from two. Table~\ref{tab:e1-cqr-lambdas} explains the
  underlying mechanism.}
  \label{tab:e1-cqr-excess}
  \small
  \setlength{\tabcolsep}{4.5pt}

  \resizebox{\linewidth}{!}{%
  \begin{tabular}{@{}ll*{4}{r}@{}}
    \toprule
    \hb
    & &
    \multicolumn{2}{c}{\normscore} &
    \multicolumn{2}{c}{\secscore} \\

    \cmidrule(lr){3-4}
    \cmidrule(lr){5-6}
    \hb
    \kilth{Grounder} &
    \kilth{Dataset} &
    \kilth{$1-\alpha=0.80$} &
    \kilth{$1-\alpha=0.90$} &
    \kilth{$1-\alpha=0.80$} &
    \kilth{$1-\alpha=0.90$} \\

    \midrule

    QD-DETR &
    Charades-STA &
    $-0.8\%$ &
    $+1.6\%$ &
    $-0.2\%$ &
    $+0.0\%$ \\

    QD-DETR &
    ActivityNet-Captions &
    $+0.8\%$ &
    $+0.9\%$ &
    $+1.0\%$ &
    $+0.4\%$ \\

    QD-DETR &
    QVHighlights &
    $+5.0\%$ &
    $+1.3\%$ &
    $+4.6\%$ &
    $+0.6\%$ \\

    \addlinespace

    Qwen &
    Charades-STA &
    $+1.6\%$ &
    $-1.5\%$ &
    $+4.3\%$ &
    $+0.3\%$ \\

    Qwen &
    ActivityNet-Captions &
    $-4.2\%$ &
    $-3.0\%$ &
    $-1.8\%$ &
    $-4.9\%$ \\

    Qwen &
    QVHighlights &
    $-9.1\%$ &
    $-8.6\%$ &
    $-6.6\%$ &
    $-7.2\%$ \\

    \bottomrule
  \end{tabular}}
\end{table}

\begin{table}[t]
  \centering
  \caption{\textbf{Calibrated boundary parameters in the two-parameter variant.}
  Mechanism behind Table~\ref{tab:e1-cqr-excess}: separately calibrated
  $\hat{\lambda}_{\text{start}}$ and $\hat{\lambda}_{\text{end}}$ for
  \normscore\ and \secscore\ at both target coverages. QD-DETR calibrates both
  boundaries to similar values throughout (ratio $0.9$--$1.6\times$), indicating
  that a single parameter is sufficient. In contrast, Qwen's start boundary
  collapses to nearly zero in seven of the twelve cases shown, already covering
  the true start for almost every calibration example, with nearly all boundary
  error concentrated at the end.}
  \label{tab:e1-cqr-lambdas}
  \tiny
  \setlength{\tabcolsep}{3pt}

  \resizebox{\linewidth}{!}{%
  \begin{tabular}{@{}ll*{4}{r}*{4}{r}@{}}
    \toprule
    \hb
    & &
    \multicolumn{4}{c}{$1-\alpha=0.80$} &
    \multicolumn{4}{c}{$1-\alpha=0.90$} \\

    \cmidrule(lr){3-6}
    \cmidrule(lr){7-10}
    \hb
    & &
    \multicolumn{2}{c}{\normscore} &
    \multicolumn{2}{c}{\secscore} &
    \multicolumn{2}{c}{\normscore} &
    \multicolumn{2}{c}{\secscore} \\

    \cmidrule(lr){3-4}
    \cmidrule(lr){5-6}
    \cmidrule(lr){7-8}
    \cmidrule(lr){9-10}
    \hb
    \kilth{Grounder} &
    \kilth{Dataset} &
    \kilth{$\hat{\lambda}_{\text{start}}$} &
    \kilth{$\hat{\lambda}_{\text{end}}$} &
    \kilth{$\hat{\lambda}_{\text{start}}$} &
    \kilth{$\hat{\lambda}_{\text{end}}$} &
    \kilth{$\hat{\lambda}_{\text{start}}$} &
    \kilth{$\hat{\lambda}_{\text{end}}$} &
    \kilth{$\hat{\lambda}_{\text{start}}$} &
    \kilth{$\hat{\lambda}_{\text{end}}$} \\

    \midrule

    QD-DETR &
    Charades-STA &
    1.058 & 1.282 &
    8.96 & 9.41 &
    1.621 & 1.860 &
    14.88 & 14.87 \\

    QD-DETR &
    ActivityNet-Captions &
    1.516 & 1.737 &
    22.33 & 28.72 &
    2.437 & 3.820 &
    40.78 & 50.22 \\

    QD-DETR &
    QVHighlights &
    2.031 & 1.883 &
    49.85 & 48.71 &
    3.243 & 3.564 &
    79.55 & 73.36 \\

    \addlinespace

    Qwen &
    Charades-STA &
    0.000 & 3.310 &
    0.00 & 18.48 &
    0.032 & 4.053 &
    0.20 & 22.20 \\

    Qwen &
    ActivityNet-Captions &
    0.000 & 10.167 &
    0.00 & 64.13 &
    0.756 & 15.359 &
    5.44 & 78.07 \\

    Qwen &
    QVHighlights &
    0.000 & 15.808 &
    0.00 & 101.20 &
    0.432 & 19.600 &
    3.80 & 115.40 \\

    \bottomrule
  \end{tabular}}
\end{table}

\subsection{Exchangeability Diagnostic for the Coverage-Validity Study}
\label{app:exchangeability}
\cref{sec:Results} reports that QD-DETR's realized coverage on QVHighlights ($0.879 \pm 0.016$ at $1-\alpha=0.90$) runs measurably below the
$k/(n+1)$ floor Proposition ~\ref{prop:beta} predicts
under i.i.d.\ calibration rows. That floor's derivation rests on
exchangeability (Assumption \ref{a:exch}) between the calibration and test draws, so we test the frozen split.

\emph{Procedure.} For a given (grounder, score family), we compute the
two-sample Kolmogorov--Smirnov statistic between the score's empirical
distribution on the frozen calibration pool and on the frozen test pool.
Significance is assessed by a permutation test rather than the classical
asymptotic KS $p$-value, since the latter assumes independent rows,
while multiple queries in QVHighlights share a source video ($\approx
3.3$ rows per source video in both pools here). We therefore permute at
the video level: pool every (video, score) pair across both frozen
pools, repeatedly reshuffle unit labels into pseudo-calibration/pseudo-test
groups of the original unit counts (every row for a unit moves together),
and recompute the KS statistic on each reshuffled draw over 2000
permutations. The $p$-value is $(\#\text{draws with KS} \geq
\text{observed} + 1)/2001$. This tests Assumption \ref{a:exch} at the same
video-level granularity as the frozen split. As a
check on the procedure's power at these sample sizes, the same test
applied to a held-out split with a deliberately injected shift reliably
rejects at this permutation budget, so null results reflect an
absence of detectable difference.

\emph{Result.} Neither grounder shows a significant difference between
its calibration and test pools on QVHighlights, under either score
family (Table~\ref{tab:qvh-exchangeability}).

\begin{table}[t]
  \centering
  \caption{\textbf{Exchangeability check on QVHighlights.}
  Video-level permutation test comparing each grounder's frozen
  QVHighlights calibration and test pools for both score families.
  No grounder--family pair exhibits a statistically significant difference.}
  \label{tab:qvh-exchangeability}
 \footnotesize
  \begin{tabular}{@{}llS[table-format=1.3]S[table-format=1.2]@{}}
    \toprule
    \hb
    \kilth{Grounder} &
    \kilth{Family} &
    \kilth{KS(cal, test)} &
    \kilth{Permutation $p$} \\
    \midrule

    QD-DETR &
    \normscore &
    0.043 &
    0.52 \\

    Qwen &
    \normscore &
    0.049 &
    0.39 \\

    QD-DETR &
    \secscore &
    0.065 &
    0.11 \\

    Qwen &
    \secscore &
    0.066 &
    0.08 \\

    \bottomrule
  \end{tabular}
\end{table}

\secscore\'s unnormalized score shows a consistently smaller $p$-value than
NORM's for both grounders, the same direction as the covariate-shift
diagnostic's finding (Table \ref{tab:e4a}) that
length-normalization absorbs part of any residual covariate imbalance
between pools, here within one dataset's own split rather than across
two datasets. Notably, the smaller of each family's two $p$-values
belongs to Qwen, not QD-DETR, in both cases: the grounder \emph{without}
a marginal coverage problem shows the larger (still non-significant)
split-level divergence. Even without
adjusting for testing across two grounders and two families, the
smallest of the four $p$-values ($0.08$) clears a conservative
Bonferroni threshold for four comparisons ($0.05/4=0.0125$) with room to
spare.

Ruling out the split points instead to the grounder-specific
conditional-coverage gap already measured in
Table~\ref{tab:e3}:  %(Section \ref{subsec:E3}): 
QD-DETR undercovers
QVHighlights' Long and Multi-window strata by $0.148$ and $0.127$
respectively, against Qwen's $0.008$ and $0.046$ on the same
strata, which is similar to the grounder split the marginal numbers show. We
attribute the marginal coverage gap to this subgroup effect, amplified by QVHighlights'
small calibration pool.

The same procedure applied to Charades-STA (\secscore, QD-DETR, a
dataset/grounder pair that is as on-target) returns
$\mathrm{KS}=0.033$, $p=0.51$, consistent with the QVHighlights results
reflecting genuinely exchangeable splits.

\paragraph{Efficiency detail}
Following the discussion in \Cref{sec:Results}, table~\ref{tab:e1-qwen-refusal} isolates Qwen's calibrated widening from refusal-driven inflation.

Table \ref{tab:e1-uncalibrated} reports the same nested family
and evaluation engine as COVER, with $\lambda$ fixed by hand instead of
calibrated. No single value is adequate across datasets: $\lambda=20$s
exceeds QD-DETR's own $0.90$-target Charades-STA coverage ($0.958$) while
giving only $0.629$ on QVHighlights and, under Qwen, just $0.169$. The
spread is starkest for Qwen on QVHighlights, where realized coverage
across the six values tested ranges from $0.096$ to $1.000$ (a range of
$0.904$, nearly the full attainable interval) achieved by varying only
which round number of seconds is picked, with no data-driven step
involved. Even the largest value tested, $\lambda=160$s, which reaches
coverage $\geq 0.999$ everywhere, does so by spending $149.6$s on every
QVHighlights row regardless of how easy that row actually is (indistinguishable from a full-video fallback ($T\approx150$s)) where
COVER's calibrated region reaches comparable coverage at the $0.90$
target in $128.6$s (Table \ref{tab:e1}) and needs only $24.9$s on Charades-STA,
adapting per row rather than spending a fixed budget everywhere. This demonstrates that a trivially large fixed margin covers, just as a trivial whole-video region does, but only the calibration step delivers a specific target at a length that reflects what each row actually needs.

\begin{table}[t]
  \centering
  \caption{\textbf{Marginal versus non-refused region length for Qwen2.5-VL-7B.}
  Region length (seconds) at target coverage $1-\alpha=0.90$. Marginal length
  includes the full-video fallback on refused examples, whereas non-refused
  length isolates the effect of calibrated interval widening.}
  \label{tab:e1-qwen-refusal}
  \small
  \begin{tabular}{@{}lS[table-format=2.1]*{2}{S[table-format=3.1]}*{2}{S[table-format=3.1]}@{}}
    \toprule
    \hb
    &
    \kilth{Refusal} &
    \multicolumn{2}{c}{\normscore} &
    \multicolumn{2}{c}{\secscore} \\

    \cmidrule(lr){3-4}
    \cmidrule(lr){5-6}
    \hb
    \kilth{Dataset} &
    \kilth{(\%)} &
    \kilth{Marg.} &
    \kilth{Non-ref.} &
    \kilth{Marg.} &
    \kilth{Non-ref.} \\

    \midrule

    Charades-STA &
    10.4 &
    25.4 &
    24.5 &
    26.1 &
    25.3 \\

    ActivityNet-Captions &
    2.9 &
    63.9 &
    63.5 &
    66.6 &
    66.2 \\

    QVHighlights &
    2.6 &
    129.1 &
    128.6 &
    133.5 &
    133.1 \\

    \bottomrule
  \end{tabular}
\end{table}

\section{Additional Experiments}
\label{app:extra}
We start with additional grounders' coverage and efficiency evaluation. Then we
report in full the four further studies summarized in \cref{sec:Results}.
Finally, we present a downstream-utility experiment.

\paragraph{Additional grounders coverage and efficiency}

To test whether COVER's guarantee and qualitative efficiency patterns are
specific to QD-DETR and Qwen or hold more broadly across other  localizers, we
add TR-DETR~\cite{sun_zhou2024tr}, integrated via the same Lighthouse framework as QD-DETR with no additional training. We evaluate on InternVL3.5-8B \cite{wang2025internvl3_5} as the second black-box VLM grounder. Table~\ref{tab:internvl} reports region length at matched target coverage. We then evaluate on SRAM \cite{ma2024beyond}, a grounder that \emph{reports} uncertainty by training. 

\paragraph{Protocols}
\paragraph{Risk control: RCPS miss-rate control}
We select $\lambda$ via the risk-controlling procedure of \cite{bates2021rcps}: treat the containment loss $\mathcal{L}_\lambda(X,\Igt)=\mathbf{1}[\score(X,\Igt)>\lambda]$ as nonincreasing in $\lambda$ (Assumption \ref{a:nested}), form a point-wise upper confidence bound on its risk from the calibration scores via a Hoeffding concentration inequality \cite{boucheron2013concentration}, and take $\hat\lambda$ as the smallest threshold whose bound stays at or below the target miss-rate $\alpha$ for every larger threshold. This is a materially different calibration step from the plain $(1-\alpha)$
quantile of Eq. (\ref{eq:qhat}): it trades away some region tightness for a high-probability guarantee over the calibration draw itself, rather than only in expectation.
We run this procedure at $\alpha=0.1$ with confidence $1-\delta=0.9$, for both base grounders, on every dataset and applicable score family. Three quantities are reported. The first is the realized miss-rate of the resulting $\hat\lambda$ on the frozen test set, which should not exceed $\alpha$. The second, using the video-level resampling protocol described in the main paper, is the fraction of resampled calibration draws whose $\hat\lambda$ would violate the target miss-rate on the unresampled test set. This operationalizes the $\delta$ confidence statement empirically, since Proposition \ref{prop:rcps}'s guarantee is itself a probability over the calibration draw. The third compares $\hat\lambda$ and its mean region length against the plain split-conformal $\hat\lambda$ from the coverage-validity and efficiency studies at the same nominal $\alpha$, showing what the stronger, high-probability guarantee costs in region size relative to the tighter but differently-scoped marginal one.

\paragraph{Conditional coverage}
The marginal coverage validated earlier is an average, and an average can hide a
badly undercovered subpopulation behind another that overcovers. Since exact
attribute-free conditional coverage is unattainable distribution-free, this study evaluates
the two practical substitutes the paper adopts. The diagnostic half calibrates one
threshold on the full pool and reports coverage within terciles of the true event
length, and, on QVHighlights, separately for single- versus multi-window ground
truth, because its strata come from the truth, it is evaluation-only. The Mondrian
half instead implements the group-conditional calibration of
Proposition \ref{prop:mondrian}, fitting a separate threshold per stratum with
strata defined by the model's own predicted length, so that they are available
before the answer is known.

\paragraph{Cross-dataset transfer}
The studies so far keep calibration and test within one dataset. This one drops that
restriction and calibrates on one dataset while testing on another, exercising the
exchangeability assumption rather than assuming it. We proceed in three steps. A
covariate-shift diagnostic comes first: a two-sample statistic asks whether the two
datasets differ on candidate covariates and, separately, whether that difference
survives into the conformal score, which the normalization may already absorb in
part and both are read against a same-dataset noise floor. We then measure unweighted
transfer, calibrating on one dataset and evaluating on the other's test pool,
with a baseline that calibrates on the test set's own dataset.
Finally we apply a weighted correction, but not the usual one, rather than fit a
continuous per-test-point density ratio, we bucket on video duration with a hard
cutoff and calibrate a separate threshold per bucket. Unlike true event length, duration needs no ground truth, so the rule could
apply to an unlabeled query, and unlike predicted length it is not already baked
into the normalized score, which divides by predicted length directly. No reweighting scheme, simplified
or literal, can correct shift where the calibration dataset has no support as it
has nothing to draw on.

\paragraph{Calibration-set size}
Theorem \ref{thm:eff} predicts the calibrated threshold converges to its population value as calibration size grows, with estimation noise shrinking at the standard $\sqrt{n}$ rate. This study checks that directly and gives a practical read on how much labeled calibration data is actually needed before coverage stabilizes. At a fixed target ($\alpha=0.1$), we repeatedly subsample smaller calibration pools and report how the spread of realized coverage across resamples changes as subsample size grows.

\paragraph{Downstream utility of the certified width}
\label{sec:e-utility}
The marginal guarantee is population-level, and the per-example width is
where individual-level information lives. We test this post-hoc on the
same calibrated $\hat{\lambda}$ that Table \ref{tab:e1} uses. For every non-refused test row we pair the region width at the
$1-\alpha=0.90$ operating point with the base grounder's point-prediction IoU against ground truth, and separately run a selective-prediction rule that
retains a row only when its certified width is at or below a budget $\tau$,
sweeping $\tau$ over the width distribution.

\subsection{Additional Grounders}
\label{app:additional-grounders}

\begin{figure}[t]
    \centering
    \kilgraphics{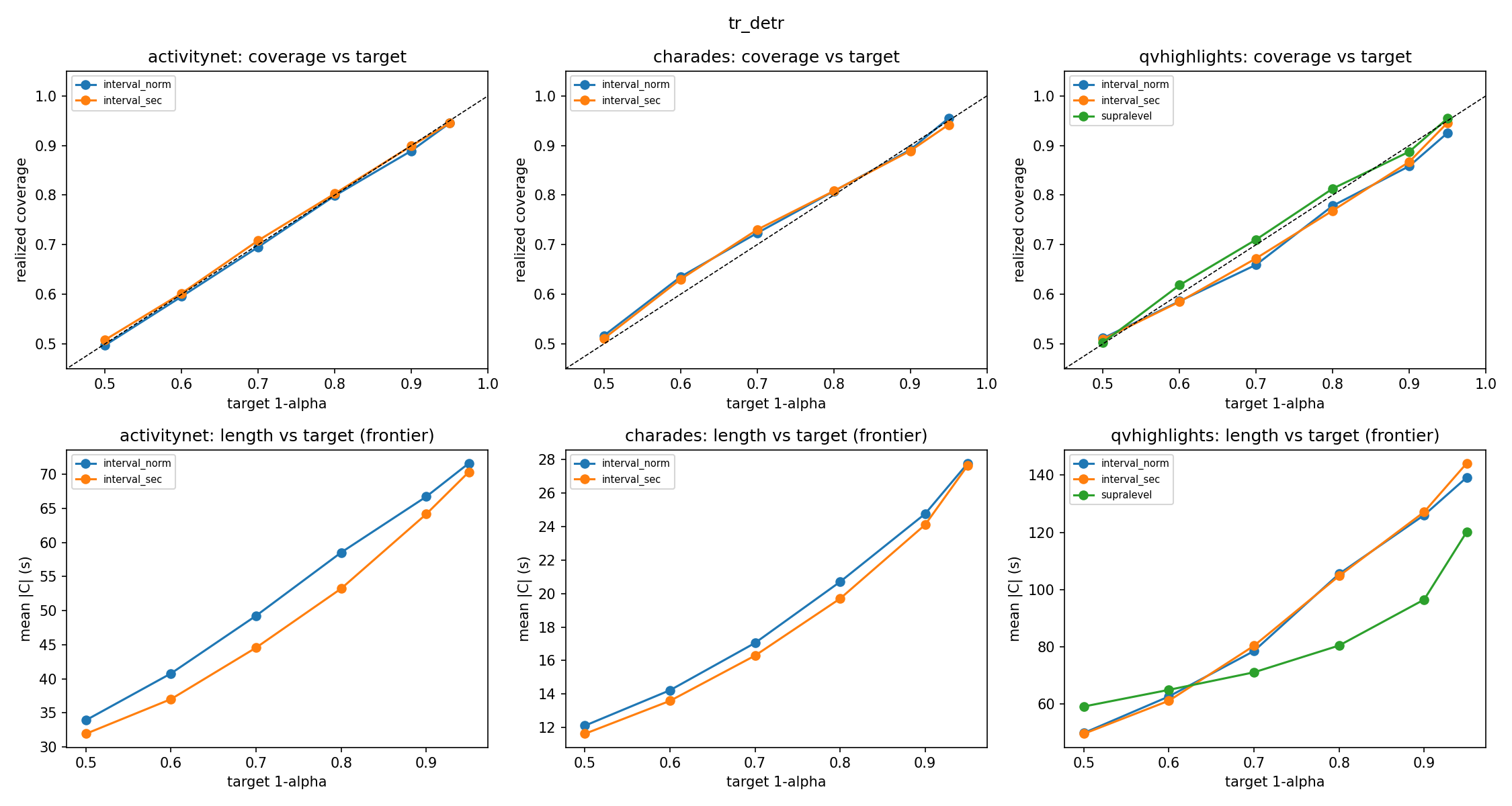}
    \caption{\textbf{Coverage validity and efficiency (TR-DETR).}
    Target coverage is achieved, confirming COVER applies to various grounders in the DETR grounders family.}
    \label{fig:e0-tr-detr}
\end{figure}

\begin{figure}[t]
    \centering
    \kilgraphics{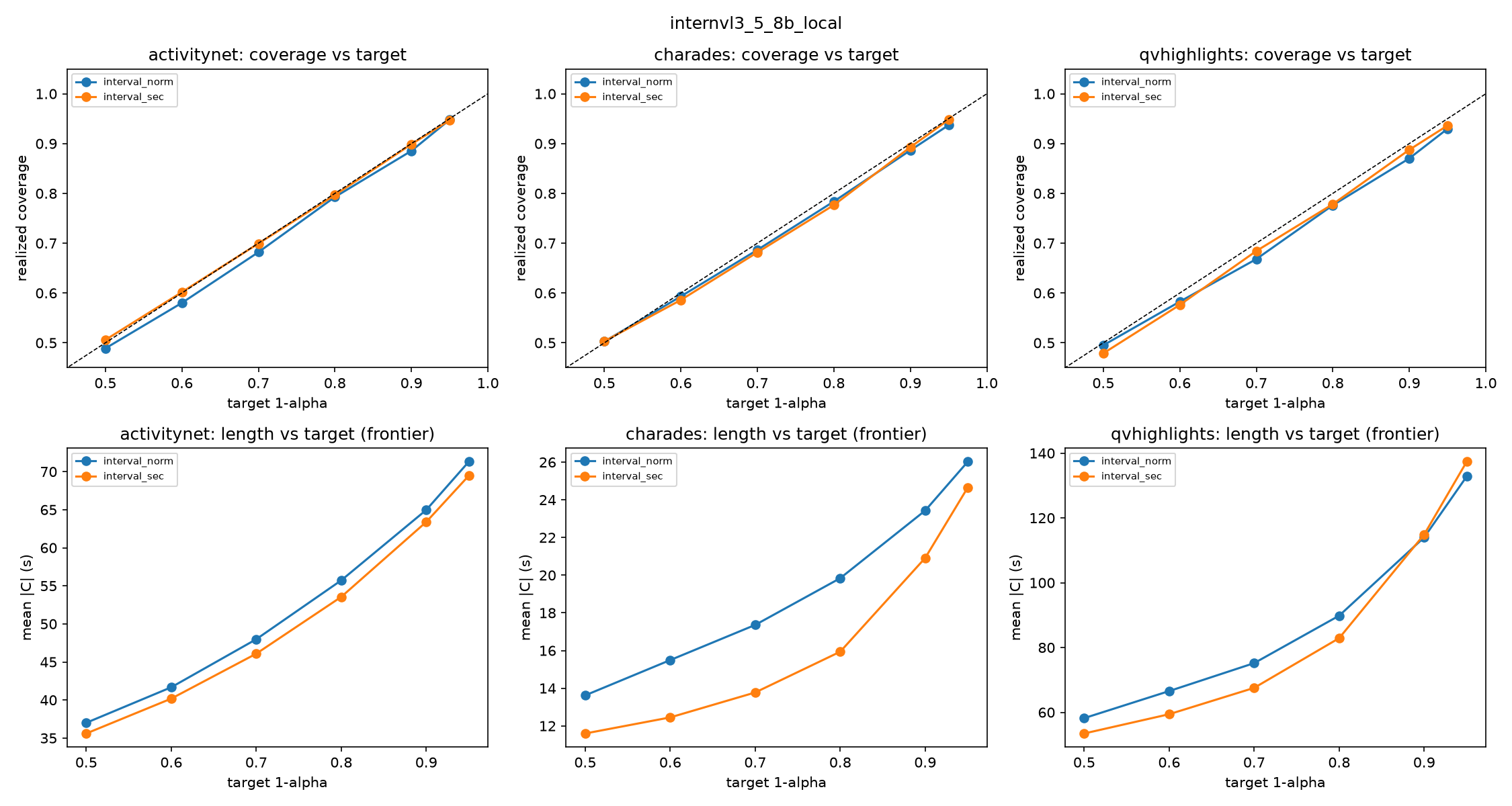}
    \caption{\textbf{Coverage validity and efficiency (InternVL3.5-8B).}
    Target coverage is achieved across the $1-\alpha$ grid, confirming COVER also applies to various black-box VLM grounders.}
    \label{fig:internvl}
\end{figure}

\paragraph{Coverage validity} For TR-DETR, realized coverage tracks target across the full
$1-\alpha \in \{0.50,\ldots,0.95\}$ grid on all three datasets and both
score families (Figure~\ref{fig:e0-tr-detr}), with all shortfalls being within the combined test and calibration sampling tolerance. Similarly for InternVL, target coverage across the $1-\alpha \in \{0.50,\ldots,0.95\}$
grid on Charades-STA, QVHighlights, and on ActivityNet-Captions is achieved (shown in Figure \ref{fig:internvl}).

\paragraph{Region length} Tables \ref{tab:tr-detr} and \ref{tab:internvl} report length at
matched target coverage.

\emph{Charades-STA crossing does not replicate on TR-DETR.} QD-DETR shows significant reversal between targets (\secscore\ shorter at $1-\alpha=0.80$, \normscore\ shorter at $0.90$). TR-DETR shows no such
reversal: \secscore\ is significantly shorter at \emph{both} targets
($+1.00$s, CI $[0.77,1.23]$ at $0.80$; $+0.65$s, CI $[0.41,0.92]$ at $0.90$). Figure \ref{fig:e0-tr-detr} establishes \secscore\ is shorter for all other targets as well on ActivityNet and Charades.

\emph{The QVHighlights \supralevel\ efficiency advantage holds, and is substantially larger for TR-DETR.} \supralevel\ is significantly shorter than both interval
families at both targets for TR-DETR, matching QD-DETR qualitatively. The magnitude differs considerably: \supralevel\  is $10.5\% $ and $13.4\%$ shorter than the interval-family average for QD-DETR at $1-\alpha=0.80$ and $0.90$, versus $23.5\%$ and $23.8\%$ for TR-DETR, roughly double. Shown in Figure \ref{fig:e0-tr-detr}, this supremacy does not hold on looser targets (0.5 and 0.6). TR-DETR's own
\normscore-vs-\secscore\ comparison on QVHighlights is a statistical tie at \emph{both} targets ($+0.71$s, CI $[-0.91,2.28]$ at $0.80$; $-1.09$s,
CI $[-2.67,0.43]$ at $0.90$), whereas QD-DETR's tie holds only at $0.90$ (its own $0.80$-target gap is significant). ActivityNet-Captions shows the
same qualitative pattern as QD-DETR at both targets (\secscore\  significantly shorter), with a larger absolute gap at $1-\alpha=0.80$ ($5.29$s vs.\ QD-DETR's $2.23$s).

\emph{\secscore\ is more efficient for InternVL.} Table \ref{tab:internvl} shows that \secscore\ is significantly shorter than \normscore\ in five of six
(dataset, target) combinations (Charades-STA and ActivityNet-Captions at both
targets, QVHighlights at $1-\alpha=0.80$). The sixth, QVHighlights at
$1-\alpha=0.90$, is a genuine statistical tie ($-0.89$s, CI $[-2.74,0.88]$). This is also evident for other target shown in Figure \ref{fig:internvl}
This is the opposite pattern from Qwen2.5-VL, whose results
favor \normscore\ in most comparisons, which indicates the identity of the more
efficient family is not a property of grounder class (VLM vs.\ trained
localizer) but varies by individual model.

\paragraph{Training for uncertainty does not produce coverage}
The most direct test of whether a post-hoc wrapper is needed is to point it at a
grounder that already models its own uncertainty. SRAM \cite{ma2024beyond} is a
deep-evidential-regression grounder that reports boundary uncertainty by
training. Converting its own predicted boundary uncertainty into a Gaussian
region at nominal $1-\alpha$, in its own best configuration and on the dataset it
was trained on, realizes $0.660$ against the $0.80$ target and $0.716$ against
$0.90$ (Table~\ref{tab:e1-SRAM}): a shortfall of $14$ and $18$ points on a
quantity the method is built to deliver. The nominal level is not a guarantee but
a label. Wrapping the identical predictions in \cover{} realizes coverage within
$0.005$ of every target, at up to $50\%$ more length. Architecture-level uncertainty modeling produces an uncertainty estimate, and distribution-free calibration is what turns an estimate
into a guarantee.

\begin{table*}[t]
  \centering
  \caption{\textbf{TR-DETR region length at matched target coverage.}
  Region length (seconds) with 95\% confidence intervals at matched target
  coverage. The shortest statistically significant score family for each
  dataset and target coverage is shown in \textbf{bold}. Blank entries indicate
  that the score family is not applicable.}
  \label{tab:tr-detr}
  \small
  \setlength{\tabcolsep}{3.5pt}

  \resizebox{\textwidth}{!}{%
  \begin{tabular}{@{}l*{3}{S[table-format=3.1]c}*{3}{S[table-format=3.1]c}@{}}
    \toprule
    \hb
    &
    \multicolumn{6}{c}{$1-\alpha=0.80$} &
    \multicolumn{6}{c}{$1-\alpha=0.90$} \\

    \cmidrule(lr){2-7}
    \cmidrule(lr){8-13}
    \hb
    &
    \multicolumn{2}{c}{\normscore} &
    \multicolumn{2}{c}{\secscore} &
    \multicolumn{2}{c}{\supralevel} &
    \multicolumn{2}{c}{\normscore} &
    \multicolumn{2}{c}{\secscore} &
    \multicolumn{2}{c}{\supralevel} \\

    \cmidrule(lr){2-3}
    \cmidrule(lr){4-5}
    \cmidrule(lr){6-7}
    \cmidrule(lr){8-9}
    \cmidrule(lr){10-11}
    \cmidrule(lr){12-13}
    \hb
    \kilth{Dataset} &
    \kilth{Len.} & \kilth{CI} &
    \kilth{Len.} & \kilth{CI} &
    \kilth{Len.} & \kilth{CI} &
    \kilth{Len.} & \kilth{CI} &
    \kilth{Len.} & \kilth{CI} &
    \kilth{Len.} & \kilth{CI} \\

    \midrule

    Charades-STA &
    20.7 & {\scriptsize[20.16,21.22]} &
    \bfseries 19.7 & {\scriptsize[19.35,20.02]} &
    {--} & {--} &
    24.7 & {\scriptsize[24.15,25.33]} &
    \bfseries 24.1 & {\scriptsize[23.71,24.48]} &
    {--} & {--} \\

    ActivityNet-Captions &
    58.5 & {\scriptsize[57.01,60.02]} &
    \bfseries 53.2 & {\scriptsize[52.19,54.26]} &
    {--} & {--} &
    66.7 & {\scriptsize[65.02,68.44]} &
    \bfseries 64.1 & {\scriptsize[62.78,65.47]} &
    {--} & {--} \\

    QVHighlights &
    105.4 & {\scriptsize[102.90,107.95]} &
    104.7 & {\scriptsize[103.07,106.43]} &
    \bfseries 80.4 & {\scriptsize[77.15,83.48]} &
    126.0 & {\scriptsize[123.83,128.08]} &
    127.1 & {\scriptsize[125.70,128.50]} &
    \bfseries 96.4 & {\scriptsize[93.52,99.46]} \\

    \bottomrule
  \end{tabular}}
\end{table*}

\begin{table*}[t]
  \centering
  \caption{\textbf{InternVL3.5-8B region length at matched target coverage.}
  Region length (seconds) with 95\% confidence intervals at matched target
  coverage. The shortest statistically significant score family for each
  dataset and target coverage is shown in \textbf{bold}. No
  \supralevel\ column is reported because this grounder exposes no
  saliency or relevance signal.}
  \label{tab:internvl}
  \small
  \setlength{\tabcolsep}{3.5pt}

  \resizebox{\linewidth}{!}{%
  \begin{tabular}{@{}l*{2}{S[table-format=3.1]c}*{2}{S[table-format=3.1]c}@{}}
    \toprule
    \hb
    &
    \multicolumn{4}{c}{$1-\alpha=0.80$} &
    \multicolumn{4}{c}{$1-\alpha=0.90$} \\

    \cmidrule(lr){2-5}
    \cmidrule(lr){6-9}
    \hb
    &
    \multicolumn{2}{c}{\normscore} &
    \multicolumn{2}{c}{\secscore} &
    \multicolumn{2}{c}{\normscore} &
    \multicolumn{2}{c}{\secscore} \\

    \cmidrule(lr){2-3}
    \cmidrule(lr){4-5}
    \cmidrule(lr){6-7}
    \cmidrule(lr){8-9}
    \hb
    \kilth{Dataset} &
    \kilth{Len.} & \kilth{CI} &
    \kilth{Len.} & \kilth{CI} &
    \kilth{Len.} & \kilth{CI} &
    \kilth{Len.} & \kilth{CI} \\

    \midrule

    Charades-STA &
    19.8 & {\scriptsize[19.21,20.48]} &
    \bfseries 15.9 & {\scriptsize[15.62,16.25]} &
    23.4 & {\scriptsize[22.80,24.07]} &
    \bfseries 20.9 & {\scriptsize[20.56,21.27]} \\

    ActivityNet-Captions &
    55.7 & {\scriptsize[54.27,57.10]} &
    \bfseries 53.5 & {\scriptsize[52.43,54.62]} &
    65.0 & {\scriptsize[63.35,66.59]} &
    \bfseries 63.4 & {\scriptsize[62.07,64.75]} \\

    QVHighlights &
    89.8 & {\scriptsize[86.74,92.96]} &
    \bfseries 82.9 & {\scriptsize[80.72,85.02]} &
    114.0 & {\scriptsize[111.22,116.61]} &
    114.9 & {\scriptsize[113.20,116.53]} \\

    \bottomrule
  \end{tabular}}
\end{table*}

\begin{table}[t]
  \centering
  \caption{\textbf{SRAM versus SRAM + COVER.} SRAM's evidential uncertainty
    converted directly into prediction regions at nominal coverage $1-\alpha$
    (without a calibration step; SRAM trained on Charades-STA) compared with
    SRAM + COVER. \colorbox{underc}{Shaded} cells indicate coverage more than
    $0.02$ below the target.}
  \label{tab:e1-SRAM}
  \small
  \setlength{\tabcolsep}{4.5pt}
  \begin{tabular}{@{}l*{2}{r}*{2}{r}@{}}
    \toprule
    \hb
    & \multicolumn{2}{c}{$1-\alpha=0.80$} & \multicolumn{2}{c}{$1-\alpha=0.90$} \\
    \cmidrule(lr){2-3}\cmidrule(lr){4-5}
    \hb
    \kilth{Method} & \kilth{Cov.} & \kilth{Len.} & \kilth{Cov.} & \kilth{Len.}\\
    \midrule
    SRAM & \colorbox{underc}{0.660} & 14.2\,s & \colorbox{underc}{0.716} & 15.7\,s \\
    \addlinespace
    SRAM \emph{+ \cover{}} (\normscore) & 0.797 & 19.3\,s & 0.895 & 23.5\,s \\
    SRAM \emph{+ \cover{}} (\secscore)  & 0.796 & 18.4\,s & 0.898 & 23.2\,s \\
    \bottomrule
  \end{tabular}
\end{table}

\subsection{Risk Control: RCPS Miss-Rate Control}
\label{subsec:E2}

\begin{table*}[t]
  \centering
  \caption{\textbf{RCPS versus split conformal.}
  RCPS ($\alpha=0.1$, $\delta=0.1$) compared with plain split conformal at the
  same nominal $\alpha$ for both grounders. \emph{Viol.} denotes the fraction
  of 50 resampled calibration draws whose calibrated
  $\hat{\lambda}$ violated the target miss rate on the frozen test set. The
  RCPS guarantee bounds this fraction by $\delta=0.10$. RCPS trades a modest
  increase in region length ($\Delta$) for a stronger high-probability
  guarantee.}
  \label{tab:e2}
  \scriptsize
  \setlength{\tabcolsep}{5pt}

  \resizebox{\linewidth}{!}{%
  \begin{tabular}{@{}lll
      S[table-format=3.3]
      S[table-format=1.3]
      S[table-format=3.1]
      S[table-format=1.2]
      S[table-format=3.1]
      S[table-format=+2.1]@{}}
    \toprule
    \hb
    \kilth{Grounder} &
    \kilth{Dataset} &
    \kilth{Family} &
    \kilth{$\hat{\lambda}_{\mathrm{RCPS}}$} &
    \kilth{Cov.} &
    \kilth{Len. (s)} &
    \kilth{Viol.} &
    \kilth{Len.\textsubscript{split}} &
    \kilth{$\Delta$ (\%)} \\

    \midrule

    \multirow{7}{*}{QD-DETR}
      & Charades-STA
      & \normscore
      & 2.035
      & 0.928
      & 26.1
      & 0.00
      & 24.4
      & +7.0 \\

      & Charades-STA
      & \secscore
      & 16.419
      & 0.918
      & 26.0
      & 0.00
      & 24.9
      & +4.3 \\

      & ActivityNet-Captions
      & \normscore
      & 3.361
      & 0.906
      & 64.9
      & 0.00
      & 63.6
      & +2.1 \\

      & ActivityNet-Captions
      & \secscore
      & 48.777
      & 0.912
      & 63.2
      & 0.00
      & 61.7
      & +2.4 \\

      & QVHighlights
      & \normscore
      & 4.778
      & 0.933
      & 140.4
      & 0.00
      & 129.1
      & +8.7 \\

      & QVHighlights
      & \secscore
      & 93.350
      & 0.920
      & 139.8
      & 0.00
      & 128.6
      & +8.7 \\

      & QVHighlights
      & \supralevel
      & 12.163
      & 0.936
      & 128.1
      & 0.00
      & 111.6
      & +14.8 \\

    \addlinespace

    \multirow{6}{*}{Qwen}
      & Charades-STA
      & \normscore
      & 3.733
      & 0.923
      & 26.1
      & 0.00
      & 25.4
      & +2.9 \\

      & Charades-STA
      & \secscore
      & 20.700
      & 0.925
      & 27.2
      & 0.02
      & 26.1
      & +4.2 \\

      & ActivityNet-Captions
      & \normscore
      & 11.516
      & 0.908
      & 64.9
      & 0.00
      & 63.9
      & +1.5 \\

      & ActivityNet-Captions
      & \secscore
      & 67.180
      & 0.912
      & 67.5
      & 0.00
      & 66.6
      & +1.4 \\

      & QVHighlights
      & \normscore
      & 19.304
      & 0.941
      & 134.5
      & 0.00
      & 129.1
      & +4.2 \\

      & QVHighlights
      & \secscore
      & 112.100
      & 0.931
      & 139.8
      & 0.00
      & 133.5
      & +4.7 \\

    \bottomrule
  \end{tabular}}
\end{table*}

\begin{table*}[t]
  \centering
  \caption{\textbf{Conditional coverage, \normscore\ only, at the $1-\alpha=0.90$ target.} Diagnostic: one marginal $\lambda$, stratified post hoc by
true event length (evaluation-only). Mondrian: a separate $\lambda$ per stratum of
the model's own \emph{predicted} length. Strata are length terciles. QVHighlights
adds single- vs.\ multi-window. \colorbox{underc}{Shaded} coverage cells fall more
than $0.02$ below target.
Table~\ref{tab:e3-sec} repeats this diagnostic for \secscore. Qwen's Mondrian calibration additionally produces a fourth, \emph{refused} stratum for every dataset, which are rows with no predicted length to assign a tercile from, trivially covered at $1.000$ by the refusal policy and omitted here since it is not one of the three predicted length strata.}
  \label{tab:e3}
  \LARGE
  \setlength{\tabcolsep}{4.5pt}
  \resizebox{\linewidth}{!}{%
  \begin{tabular}{@{}ll *{8}{r}@{}}
    \toprule
    \hb & & \multicolumn{2}{c}{\gh{QD-DETR diag.}} & \multicolumn{2}{c}{\gh{QD-DETR Mondrian}} & \multicolumn{2}{c}{\gh{Qwen diag.}} & \multicolumn{2}{c}{\gh{Qwen Mondrian}} \\
    \cmidrule(lr){3-4}\cmidrule(lr){5-6}\cmidrule(lr){7-8}\cmidrule(lr){9-10}
    \hb \kilth{Dataset} & \kilth{Stratum} & \kilth{Cov.} & \kilth{Len.} & \kilth{Cov.} & \kilth{Len.} & \kilth{Cov.} & \kilth{Len.} & \kilth{Cov.} & \kilth{Len.}\\
    \midrule
    \multirow{3}{*}{Charades-STA} & Short  & 0.881 & 21.1 & 0.897 & 19.0 & 0.900 & 22.4 & 0.891 & 19.6 \\
                                    & Medium & 0.908 & 24.1 & 0.901 & 25.3 & 0.901 & 25.0 & 0.892 & 27.0 \\
                                    & Long   & 0.903 & 27.9 & 0.903 & 29.8 & 0.880 & 28.5 & 0.898 & 30.1 \\
    \addlinespace
    \multirow{3}{*}{ActivityNet-Captions} & Short  & 0.895 & 37.6 & 0.896 & 36.2 & 0.907 & 41.0 & 0.887 & 49.5 \\
                                            & Medium & 0.925 & 61.9 & 0.890 & 67.4 & 0.913 & 61.0 & 0.898 & 63.9 \\
                                            & Long   & \colorbox{underc}{0.862} & 93.3 & 0.902 & 86.4 & \colorbox{underc}{0.868} & 91.7 & 0.895 & 82.3 \\
    \addlinespace
    \multirow{3}{*}{QVHighlights} & Short  & 0.947 & 119.8 & \colorbox{underc}{0.844} & 132.0 & 0.884 & 116.5 & 0.932 & 115.9 \\
                                    & Medium & 0.942 & 134.8 & \colorbox{underc}{0.863} & 129.7 & 0.923 & 128.6 & 0.930 & 141.9 \\
                                    & Long   & \colorbox{underc}{0.752} & 131.7 & \colorbox{underc}{0.858} & 126.4 & 0.892 & 141.1 & 0.886 & 146.9 \\
    \addlinespace
    \multirow{2}{*}{QVHighlights} & Single-window & 0.935 & 131.3 & {--} & {--} & 0.925 & 128.2 & {--} & {--} \\
                                    & Multi-window  & \colorbox{underc}{0.773} & 125.0 & {--} & {--} & \colorbox{underc}{0.854} & 130.9 & {--} & {--} \\
    \bottomrule
  \end{tabular}
  }
\end{table*}

Table~\ref{tab:e2} reports the risk-controlling calibration of Proposition \ref{prop:rcps} at $\alpha=0.1$, $\delta=0.1$, against the plain split-conformal result at the same nominal $\alpha$ (from the efficiency study).
The empirical violation fraction is exactly zero in 12 of the 13 (grounder, dataset, family) combinations tested. The one exception, Qwen on Charades-STA with \secscore, shows 1 of 50 draws (0.02), still within the $\delta=0.1$ budget. 

RCPS's region length exceeds the plain split-conformal length at every combination tested, and the excess follows calibration size inversely. For QD-DETR, the increase is 2.1--2.4\% on ActivityNet-Captions ($n_\mathrm{cal}=7811$), 4.3--7.0\% on Charades-STA ($n_\mathrm{cal}=1500$), and 8.7--14.8\% on QVHighlights ($n_\mathrm{cal}=625$). Qwen shows the same ordering (1.4--1.5\%, 2.9--4.2\%, 4.2--4.7\%, respectively). This matches the Hoeffding-based slack term in the RCPS construction, which scales as $1/\sqrt{n}$: the smallest calibration pool pays the largest length cost for the stronger guarantee.

\subsection{Conditional Coverage}
\label{subsec:E3}

Table~\ref{tab:e3} reports both halves of this experiment for \normscore\ at $1-\alpha=0.90$. The diagnostic half calibrates a single marginal threshold and reports coverage stratified post-hoc by true event-length tercile and, on QVHighlights, by single- vs.\ multi-window ground truth. The Mondrian half calibrates a separate threshold per tercile of the model's own \emph{predicted} length. Because the two halves stratify the test set by different, only loosely correlated properties, we do not compare same-named strata across the two halves directly, we compare each half's own worst-covered stratum against the 0.90 target.

For QD-DETR, marginal calibration leaves a real gap on QVHighlights: the worst diagnostic stratum (Long) deviates from target by 0.148 (coverage 0.752), while the worst Mondrian stratum (Short) deviates by only 0.056 (coverage 0.844), which constitutes more than 2.5$\times$ reduction in worst-case deviation. The same pattern, smaller in magnitude, holds on ActivityNet-Captions (worst diagnostic deviation 0.038 vs.\ worst Mondrian deviation 0.010) and Charades-STA (0.019 vs.\ 0.003). For Qwen the pattern holds on Charades-STA (0.020 vs.\ 0.009) and ActivityNet-Captions (0.032 vs.\ 0.013), but not on QVHighlights, where Mondrian's worst-stratum deviation (0.032) is slightly larger than the diagnostic half's (0.023). This is the one case in which stratified calibration does not clearly reduce worst-case deviation, an effect of QVHighlights' small calibration pool (625 rows) split further into per-stratum subsets.

For completeness, we compare QVHighlights' single- vs.\ multi-window stratification, evaluated only in the diagnostic half (no Mondrian equivalent, since that half stratifies by length, not window count), shows the same qualitative failure mode: multi-window queries are undercovered relative to target (QD-DETR: 0.773, Qwen: 0.854) while single-window queries are not (0.935, 0.925).

\begin{table*}[t]
  \centering
  \caption{\textbf{Length-stratified coverage by score family.}
  \normscore\ versus \secscore\ at the $1-\alpha=0.90$ target for both
  grounders. Relative to \normscore, \secscore\ undercovers the Long stratum
  while overcovering the Short and Medium strata on Charades-STA and
  ActivityNet-Captions. QVHighlights' Long and Multi-window strata are
  additionally affected by the grounder-specific effect discussed in the text.
  Strata are defined as in Table~\ref{tab:e3}.
  \colorbox{underc}{Shaded} coverage values fall more than $0.02$ below the
  nominal target.}
  \label{tab:e3-sec}
  \footnotesize
  \setlength{\tabcolsep}{4.5pt}

  \resizebox{\linewidth}{!}{%
  \begin{tabular}{@{}ll*{8}{r}@{}}
    \toprule

   \hb & &
    \multicolumn{2}{c}{\gh{QD-DETR \normscore}} &
    \multicolumn{2}{c}{\gh{QD-DETR \secscore}} &
    \multicolumn{2}{c}{\gh{Qwen \normscore}} &
    \multicolumn{2}{c}{\gh{Qwen \secscore}} \\

    \cmidrule(lr){3-4}
    \cmidrule(lr){5-6}
    \cmidrule(lr){7-8}
    \cmidrule(lr){9-10}
    \hb
    \kilth{Dataset} &
    \kilth{Stratum} &

    \kilth{Cov.} &
    \kilth{Len.} &
    \kilth{Cov.} &
    \kilth{Len.} &
    \kilth{Cov.} &
    \kilth{Len.} &
    \kilth{Cov.} &
    \kilth{Len.} \\

    \midrule

    Charades-STA
      & Short
      & 0.881 & 21.1
      & 0.898 & 22.6
      & 0.900 & 22.4
      & 0.906 & 23.6 \\

      & Medium
      & 0.908 & 24.1
      & 0.915 & 24.9
      & 0.901 & 25.0
      & 0.919 & 25.7 \\

      & Long
      & 0.903 & 27.9
      & 0.884 & 27.1
      & 0.880 & 28.5
      & \colorbox{underc}{0.834} & 28.6 \\

    \addlinespace

    ActivityNet-Captions
      & Short
      & 0.895 & 37.6
      & 0.938 & 42.9
      & 0.907 & 41.0
      & 0.968 & 48.7 \\

      & Medium
      & 0.925 & 61.9
      & 0.948 & 59.2
      & 0.913 & 61.0
      & 0.949 & 64.7 \\

      & Long
      & \colorbox{underc}{0.862} & 93.3
      & \colorbox{underc}{0.817} & 84.4
      & \colorbox{underc}{0.868} & 91.7
      & \colorbox{underc}{0.778} & 87.8 \\

    \addlinespace

    QVHighlights
      & Short
      & 0.947 & 119.8
      & 0.937 & 123.5
      & 0.884 & 116.5
      & 0.961 & 131.5 \\

      & Medium
      & 0.942 & 134.8
      & 0.936 & 129.5
      & 0.923 & 128.6
      & 0.972 & 132.4 \\

      & Long
      & \colorbox{underc}{0.752} & 131.7
      & \colorbox{underc}{0.691} & 132.2
      & 0.892 & 141.1
      & \colorbox{underc}{0.723} & 136.6 \\

    \addlinespace

    QVHighlights
      & Single-window
      & 0.935 & 131.3
      & 0.899 & 128.8
      & 0.925 & 128.2
      & 0.937 & 133.5 \\

      & Multi-window
      & \colorbox{underc}{0.773} & 125.0
      & \colorbox{underc}{0.766} & 128.3
      & \colorbox{underc}{0.854} & 130.9
      & \colorbox{underc}{0.785} & 133.5 \\

    \bottomrule
  \end{tabular}}
\end{table*}

\begin{table}[t]
  \centering
  \caption{\textbf{Covariate-shift diagnostic.}
  Two-sample KS statistics for Charades-STA (calibration) $\rightarrow$
  ActivityNet-Captions (test) using QD-DETR. The ratio is the cross-dataset KS
  divided by the larger of the two within-dataset KS statistics. Larger
  cross-dataset KS values and ratios indicate stronger distribution shift. The
  reverse direction yields closely matching results (not shown).}
  \label{tab:e4a}
  \small
  \begin{tabular}{@{}l
      S[table-format=1.3]
      S[table-format=1.3]
      S[table-format=1.3]
      S[table-format=2.1]@{}}
    \toprule
    \hb
    \kilth{Feature} &
    \kilth{Within-cal.} &
    \kilth{Within-test} &
    \kilth{Cross} &
    \kilth{Ratio ($\times$)} \\

    \midrule

    True event length & 0.033 & 0.016 & 0.526 & 16.0 \\
    Video duration    & 0.072 & 0.054 & 0.689 & 9.5 \\
    Predicted length  & 0.041 & 0.030 & 0.574 & 13.9 \\

    \addlinespace

    Score, \normscore & 0.025 & 0.014 & 0.141 & 5.7 \\
    Score, \secscore  & 0.033 & 0.014 & 0.199 & 6.1 \\

    \bottomrule
  \end{tabular}
\end{table}

\begin{table}[t]
  \centering
  \caption{\textbf{Cross-dataset transfer.}
  QD-DETR matched (calibrated on the target dataset), unweighted transfer, and
  weighted transfer (duration-bucketed with two buckets). Gap is defined as
  matched minus transferred coverage; values closer to zero indicate better
  transfer. C$\rightarrow$A denotes calibration on Charades-STA and testing on
  ActivityNet-Captions; A$\rightarrow$C is the reverse direction.}
  \label{tab:e4b}
  \small
  \setlength{\tabcolsep}{5pt}
  \begin{tabular}{@{}ll
      S[table-format=1.2]
      S[table-format=1.3]
      S[table-format=1.3]
      S[table-format=+1.3]
      S[table-format=1.3]
      S[table-format=+1.3]@{}}
    \toprule
    \hb
    &
    &
    &
    &
    \multicolumn{2}{c}{\gh{Unweighted}} &
    \multicolumn{2}{c}{\gh{Weighted}} \\

    \cmidrule(lr){5-6}
    \cmidrule(lr){7-8}

    \hb
    \kilth{Dir.} &
    \kilth{Family} &
    \kilth{Target} &
    \kilth{Matched} &
    \kilth{Cov.} &
    \kilth{Gap} &
    \kilth{Cov.} &
    \kilth{Gap} \\

    \midrule

    C$\rightarrow$A
      & \normscore
      & 0.80
      & 0.800
      & 0.752
      & +0.048
      & 0.752
      & +0.048 \\

    C$\rightarrow$A
      & \normscore
      & 0.90
      & 0.894
      & 0.812
      & +0.083
      & 0.812
      & +0.082 \\

    C$\rightarrow$A
      & \secscore
      & 0.80
      & 0.807
      & 0.632
      & +0.175
      & 0.632
      & +0.175 \\

    C$\rightarrow$A
      & \secscore
      & 0.90
      & 0.902
      & 0.713
      & +0.189
      & 0.713
      & +0.189 \\

    \addlinespace

    A$\rightarrow$C
      & \normscore
      & 0.80
      & 0.820
      & 0.885
      & -0.065
      & 0.865
      & -0.045 \\

    A$\rightarrow$C
      & \normscore
      & 0.90
      & 0.898
      & 0.979
      & -0.081
      & 0.960
      & -0.063 \\

    A$\rightarrow$C
      & \secscore
      & 0.80
      & 0.801
      & 0.989
      & -0.188
      & 0.837
      & -0.036 \\

    A$\rightarrow$C
      & \secscore
      & 0.90
      & 0.899
      & 1.000
      & -0.101
      & 0.957
      & -0.058 \\

    \bottomrule
  \end{tabular}
\end{table}

Table~\ref{tab:e3-sec} isolates Proposition \ref{prop:lcc}'s normalization claim, on Charades-STA and ActivityNet-Captions, \secscore\ undercovers
Long relative to \normscore\ in all four (grounder, dataset) pairs
($-0.019$ to $-0.090$) and overcovers Short/Medium in all eight
comparisons ($+0.006$ to $+0.061$). Twelve of twelve stratum comparisons
matching the direction Proposition \ref{prop:lcc} predicts for an unnormalized,
seconds-valued score. QVHighlights' Long and Multi-window strata show the
same directional gap for both grounders and both families, but are
confounded there by the grounder-specific effect in the
exchangeability diagnostic.

\begin{table*}[t]
  \centering
  \caption{\textbf{Coverage stability vs.\ calibration size (in videos).} \normscore,
    target $1-\alpha=0.90$. Cells: mean\,$\pm$\,s.d.\ of realized coverage over 50
    resamples; the spread narrows as size grows, matching the $O(n^{-1/2})$ rate of
    Proposition \ref{prop:beta}.}
  \label{tab:e5}
  \small
  \setlength{\tabcolsep}{4pt}
  \resizebox{\linewidth}{!}{%
  \begin{tabular}{@{}ll ccccc@{}}
    \toprule
    \hb \kilth{Grounder} & \kilth{Dataset} & {\kilth{$n{=}20$}} & {\kilth{$n{=}50$}} & {\kilth{$n{=}100$}} & {\kilth{$n{=}200$}} & {\kilth{$n{=}400$}}\\
    \midrule
    \multirow{3}{*}{QD-DETR} & Charades-STA         & 0.910$\pm$0.042 & 0.906$\pm$0.025 & 0.900$\pm$0.018 & 0.897$\pm$0.013 & 0.899$\pm$0.007 \\
                              & ActivityNet-Captions & 0.901$\pm$0.028 & 0.897$\pm$0.017 & 0.893$\pm$0.013 & 0.896$\pm$0.007 & 0.895$\pm$0.004 \\
                              & QVHighlights         & 0.885$\pm$0.041 & 0.876$\pm$0.028 & 0.876$\pm$0.018 & \dag             & \dag             \\
    \addlinespace
    \multirow{3}{*}{Qwen}    & Charades-STA         & 0.886$\pm$0.059 & 0.893$\pm$0.028 & 0.898$\pm$0.022 & 0.895$\pm$0.014 & 0.895$\pm$0.007 \\
                              & ActivityNet-Captions & 0.889$\pm$0.037 & 0.897$\pm$0.021 & 0.892$\pm$0.014 & 0.895$\pm$0.011 & 0.895$\pm$0.006 \\
                              & QVHighlights         & 0.905$\pm$0.027 & 0.900$\pm$0.019 & 0.899$\pm$0.009 & \dag             & \dag             \\
    \bottomrule
  \end{tabular}}
  \tabnote{\dag~QVHighlights' pool does not support $n=200,400$.}
\end{table*}

\begin{table}[t]
  \centering
  \caption{\textbf{Downstream utility, selective-prediction view.} QD-DETR,
    \normscore, $1-\alpha=0.90$. Retain a row when its certified width is at or
    below a budget; columns give retention level, mean point-IoU of retained vs.\
    abstained rows, and coverage among the retained.   On Charades, retained rows better localized than abstained ones (a
    width budget improves quality). On QVHighlights this is reversed, and
    narrow-budget retained rows undercover the $0.90$ target: marginal theory
    does not guarantee coverage after conditioning on width.}
  \label{tab:e-utility-selective}
  \small
  \begin{tabular}{@{}l S[table-format=1.2] *{2}{S[table-format=1.3]} S[table-format=1.3]@{}}
    \toprule
    \hb & & \multicolumn{2}{c}{\gh{Point IoU}} & \\
    \cmidrule(lr){3-4}
    \hb \kilth{Dataset} & {\kilth{Retention}} & {\kilth{ret.}} & {\kilth{abst.}} & {\kilth{Cov.\ (ret.)}}\\
    \midrule
    \multirow{3}{*}{Charades-STA} & 0.20 & 0.497 & 0.426 & 0.881 \\
                                  & 0.50 & 0.478 & 0.402 & 0.868 \\
                                  & 1.00 & 0.440 & {--}  & 0.898 \\
    \addlinespace
    \multirow{3}{*}{QVHighlights} & 0.20 & 0.435 & 0.547 & 0.632 \\
                                  & 0.50 & 0.497 & 0.552 & 0.762 \\
                                  & 1.00 & 0.525 & {--}  & 0.879 \\
    \bottomrule
  \end{tabular}

\end{table}

\begin{table}[t]
  \centering
  \caption{\textbf{Downstream utility, correlation view.} \normscore,
    $1-\alpha=0.90$. Spearman $\rho$ between certified region width and the base
    grounder's \emph{point}-prediction IoU over non-refused test rows, with mean
    point-IoU by width quartile (Q1 narrowest, Q4 widest). Negative $\rho$ or decreasing quartiles indicates narrower regions
    correspond to better localization; the sign is dataset-dependent. $n$ is the
    non-refused count.}
  \label{tab:e-utility-corr}
  \small
  \begin{tabular}{@{}ll c *{4}{S[table-format=1.3]} S[table-format=5.0]@{}}
    \toprule
    \hb & & & \multicolumn{4}{c}{\gh{Point-IoU by width quartile}} & \\
    \cmidrule(lr){4-7}
    \hb
    \kilth{Grounder} & \kilth{Dataset} & \kilth{$\rho$} & {\kilth{Q1}} & {\kilth{Q2}} & {\kilth{Q3}} & {\kilth{Q4}} & {\kilth{$n$}}\\
    \midrule
    \multirow{3}{*}{QD-DETR} & Charades-STA         & $-0.130$ & 0.489 & 0.468 & 0.419 & 0.384 & 2220  \\
                              & ActivityNet-Captions & $+0.026$ & 0.424 & 0.424 & 0.425 & 0.446 & 11685 \\
                              & QVHighlights         & $+0.115$ & 0.446 & 0.547 & 0.535 & 0.571 & 925   \\
    \addlinespace
    \multirow{3}{*}{Qwen}    & Charades-STA         & $-0.142$ & 0.306 & 0.240 & 0.168 & 0.232 & 1989  \\
                              & ActivityNet-Captions & $-0.039$ & 0.274 & 0.221 & 0.210 & 0.251 & 11346 \\
                              & QVHighlights         & $+0.061$ & 0.053 & 0.104 & 0.192 & 0.125 & 901   \\
    \bottomrule
  \end{tabular}%
 
\end{table}

\subsection{Cross-Dataset Transfer}
\label{subsec:E4}

This experiment uses QD-DETR and the Charades-STA and ActivityNet-Captions pair. Table~\ref{tab:e4a} reports a two-sample KS statistic comparing the two datasets on three raw covariates and on the two conformal scores directly, each judged against the same-dataset (calibration-vs-test) KS as a noise floor. All three raw covariates differ far beyond this floor (9.5--16.0$\times$), confirming a substantial covariate shift between the two datasets. The conformal scores differ by a smaller but still clearly non-noise margin (5.7--6.1$\times$). \normscore's built-in length-normalization reduces the raw shift by roughly three-quarters (cross-KS 0.14--0.15 vs.\ 0.57--0.61 for raw predicted length) without eliminating it, and \secscore's unnormalized score shows only a modestly larger residual shift than \normscore's (0.199 vs.\ 0.14--0.15).

Table~\ref{tab:e4b} reports coverage under unweighted transfer and under a bucketed weighted correction, against the matched, same-dataset baseline, in both directions. Unweighted transfer degrades coverage substantially in both directions, in opposite ways: calibrating on Charades-STA and testing on ActivityNet-Captions undercovers by 4.8--18.9 points and the reverse direction overcovers by 6.5--18.8 points. The weighted correction resolves this asymmetrically. From ActivityNet-Captions to Charades-STA, it works as every gap shrinks substantially, the largest being \secscore\ at 0.80 (18.8 points to 3.6). From Charades-STA to ActivityNet-Captions, it does nothing, weighted and unweighted coverage agree to three decimal places on every row. This is a positivity limitation rather than a shortcoming of the bucketed approximation as Charades-STA's duration distribution does not extend into the range most ActivityNet-Captions test videos occupy, so there is no source-side calibration data in that region for any reweighting scheme to draw on. The C$\to$A high-duration bucket draws on just $2$ of Charades-STA's $1500$ calibration rows (triggering a fallback to the full pool), while the working A$\to$C direction draws on $2783$ of ActivityNet-Captions' $7811$ calibration rows for the bucket spanning Charades-STA's range, a gap of two orders of magnitude.

\subsection{Calibration-Set Size}
\label{subsec:E5}

Table~\ref{tab:e5} reports realized coverage for \normscore\ at $\alpha=0.1$ as calibration size grows (in units of videos). For both grounders, the mean matches the 0.90 target at every size tested, and the spread narrows as size grows. For QD-DETR on Charades-STA, from $\pm 0.042$ at 20 units to $\pm 0.007$ at 400, and similarly on ActivityNet-Captions ($\pm 0.028$ to $\pm 0.004$). Qwen shows the same pattern on both datasets. This is consistent with Theorem \ref{thm:eff}'s asymptotic rate. QVHighlights could only be tested at 20, 50, and 100 units for either grounder, since its calibration pool does not contain enough distinct videos to support the larger sizes tested on the other two datasets. Within that smaller range, the same narrowing trend holds (QD-DETR: $\pm 0.041$ to $\pm 0.018$; Qwen: $\pm 0.027$ to $\pm 0.009$).

\subsection{Downstream utility of the certified width}
\emph{Width carries an actionable difficulty signal, but only where predicted
length tracks localization quality.} Table~\ref{tab:e-utility-corr} reports the Spearman correlation between width and point IoU. On Charades-STA it is negative for both grounders ($-0.130$/$-0.142$ under \normscore, $-0.151$/$-0.128$
under \secscore), and QD-DETR's width-quartile mean IoUs fall monotonically
from the narrowest to the widest quartile ($0.489\!\to\!0.384$), which demonstrates that a narrow
certified region is a better-localized one. On the long-event datasets the
association vanishes (ActivityNet-Captions, $|\rho|\le0.10$) or reverses
(QVHighlights \normscore\, $\rho=+0.115$/$+0.061$, quartile IoU rising with width). The mechanism is that the \normscore\ width scales with predicted length, and on a long event a narrow predicted interval is typically an
under-shot prediction that misses much of the event.
The width is the deployment-relevant quantity regardless, but its usefulness is dataset-dependent. The correlations are modest in magnitude throughout, and the quartile means and the selective view separate the tails more cleanly than the full-sample rank correlation.
 
\emph{Selective prediction turns the signal into a deployment rule and
subsumes refusal handling.} Table~\ref{tab:e-utility-selective} traces the rule on QD-DETR. On Charades-STA the retained rows are better localized than the abstained ones at every budget (retaining the narrowest $20\%$: point IoU $0.497$ retained vs.\ $0.426$ abstained, narrowest $50\%$: $0.478$ vs.\ $0.402$),
so a width budget raises the average quality of the auto-processed set. On
QVHighlights the ordering flips, consistent with the inverted correlation:
narrow-budget retained rows are worse localized than the abstained ones
($0.435$ vs.\ $0.547$ at $20\%$ retention). Because a refused row takes the whole-video fallback (width $=T$), refusals are abstained automatically whenever the video is longer than the budget, the common case in
QVHighlights, so the one width mechanism
filters uncertain localizations and outright refusals together with no
special-casing. A refusal on a video already shorter than the budget is
retained, but its region is then the whole (short) video, within budget by
construction.
 
\emph{Coverage among the retained is not guaranteed.} Table \ref{tab:e-utility-selective} shows the two families
diverge as Proposition \ref{prop:lcc} predicts. Under \normscore, narrow-budget selection mildly undercovers on Charades-STA and ActivityNet-Captions ($\approx0.86$) and severely on QVHighlights ($0.591$ at
the tightest budget, recovering to the marginal $0.879$ at full retention), which is the
same QVHighlights weakness the conditional-coverage analysis already discusses. Under \secscore, narrow-budget selection instead over covers, reaching
realized coverage $1.000$ at the tightest budgets on Charades-STA and
ActivityNet-Captions for both grounders, because a fixed seconds margin is large
relative to a short prediction. This is the length-conditional behavior of Proposition \ref{prop:lcc} seen from the deployment side. The normalized
score is approximately length-conditional (up to the QVHighlights breakdown), while the unnormalized one is not.

\begin{kilrepro}
All experiments were performed on a single RTX 4090 GPU sourced through Vast.ai  platform \cite{vastai2026}. \cover{} on its own does not require significant computational capabilities. The majority of the cost lies in grounder inference, especially for local models such as Qwen. 
\end{kilrepro}

\end{document}